\documentclass[acmsmall,screen,nonacm]{acmart}

\usepackage{amsmath}
\usepackage{amsthm}
\usepackage{graphicx}
\usepackage{subcaption}
\usepackage{booktabs}
\usepackage{multirow}
\usepackage{longtable}
\usepackage{microtype}
\usepackage{tikz}
\usetikzlibrary{patterns,decorations.pathreplacing,arrows.meta}
\usepackage{placeins}
\newcommand{\Xg}{\mathcal{X}_\gamma}
\newcommand{\PW}{\waterfill}
\colorlet{red}{black}  
\newenvironment{revblock}{\begingroup\color{black}}{\endgroup}

\newcommand{\Lag}{\textsc{LAG}}
\newcommand{\E}{\mathbb{E}}
\newcommand{\Var}{\operatorname{Var}}
\newcommand{\Reward}{\operatorname{Reward}}
\newcommand{\waterfill}{\operatorname{waterfill}}
\newtheorem{theorem}{Theorem}[section]
\newtheorem{lemma}[theorem]{Lemma}
\newtheorem{proposition}[theorem]{Proposition}

\newtheorem{definition}{Definition}[section]
\newtheorem{remark}{Remark}
\newtheorem{assumption}{Assumption}[section]

\setcopyright{none}

\begin{document}

\title{\Lag: A Mean-Field Optimal Policy for Online Assignment with Reusable Resources}
\author{Keerthana Gurushankar}
\affiliation{%
  \institution{Carnegie Mellon University}
  \city{Pittsburgh}
  \state{PA}
  \country{USA}
}
\email{kgurusha@andrew.cmu.edu}
\begin{abstract}

We study online task assignment with reusable resources. This problem arises in peer reviewer assignment, skill-based call routing, rideshare matching, emergency dispatch, and crowdsourcing: Tasks arrive according to a geometric process and must be immediately assigned to servers that become temporarily busy, yielding rewards that depend on the task-server pair. The goal is to maximize long-run average reward.

We cast the problem as a contextual restless multi-armed bandit (CRMAB). We derive an occupancy-measure LP (OCC-LP) that upper-bounds the reward of any work-conserving policy. From its dual we obtain \Lag, a closed-form index policy. We introduce a scaling regime in which the number of copies of each server and the number of arrivals per step grow together. This regime yields the first deterministic fluid limit for a contextual bandit. In it, \Lag\ is asymptotically optimal.

Finally, we evaluate our algorithms on peer review datasets and call center routing. The Lagrangian index achieves 83–99\% of the LP upper bound across all load levels. Moreover, the Lagrangian index policy significantly outperforms the natural greedy assignment scheme, improving reward by up to 12.3\% on journal data and 36\% on structured synthetic instances.
\end{abstract}


\maketitle

\section{Introduction}
\label{sec:intro}

Many service systems maintain a pool of heterogeneous reusable resources (agents, servers, responders) that must be assigned to sequentially arriving jobs. In a skill-based call center, an incoming call must be routed to one of $R$ available agent pools; in a data center, an arriving job is dispatched to one of $R$ server types whose architecture may or may not suit its workload; in emergency dispatch, an incident is assigned $k$ units from a heterogeneous fleet. A large body of literature in the queueing and performance evaluation communities has studied these skill-based parallel service systems under binary compatibility: Server $r$ can serve job type $v$, or it cannot~\cite{adan2014skill,andradottir2003dynamic,wallace2005staffing}. Routing decisions in these models are constrained by the compatibility graph but are otherwise interchangeable: Any compatible assignment is as good as any other.

Real systems rarely satisfy this equivalence. In a call center, agents differ not just in which call types they can handle but also in how effectively they serve them: A specialist consistently outperforms a generalist on calls both are qualified for. In a data center, a job may run significantly faster on a server whose architecture suits its workload. In peer review, the quality of a referee report depends on how closely a reviewer's expertise matches the paper's topic. In each of these settings, the assignment decision yields a quality score $s_{vr}$, a measure of how well resource $r$ serves task type $v$. The system operator's goal is to maximize the total quality of assignments over time.

We study a discrete-time generalization of the skill-based parallel service model in which the binary compatibility graph is replaced by a real-valued score matrix $S \in \mathbb{R}^{V \times R}$. At each time step, a task of type $v$ arrives i.i.d.; it must be immediately assigned to $k$ currently available servers, each of which then becomes busy for a random bounded number of steps $D_v$ before returning to the pool. The objective is to maximize long-run average reward per step, where assigning server $r$ to a type-$v$ task yields reward $s_{vr}$. Tasks arriving when fewer than $k$ servers are free are dropped.

\subsection{Model}
\label{subsec:model}

We consider a pool of $R$ servers and a stream of tasks arriving in discrete time
$t = 1, 2, \ldots$. We are given a match score matrix $S \in \mathbb{R}^{V \times R}$,
where $s_{vr} \in \mathbb{R}$ is the reward for assigning server $r$ to a type-$v$ task. We assume $s_{vr} \ge 0$.


Encompassing these and similar problems, our model is a discrete-time M/G/R/R system specified by the following assumptions.

\begin{description}
  \item[\textnormal{(A1) Discrete-time arrivals ($\lambda_v$).}] At each step at most one
    task arrives: A task of type $v \in [V]$ arrives with probability $\lambda_v \geq 0$,
    independently across steps, and no task arrives with probability
    $1 - \sum_{v \in [V]} \lambda_v$. We call $\lambda_v$ the type-$v$ arrival rate.
  \item[\textnormal{(A2) Type-dependent service ($D_v$).}] When a type-$v$ task is
    assigned, a single service duration $D_v$ is drawn i.i.d.\ from a distribution that
    depends only on the task type $v$. All $k$ assigned servers are busy for those $D_v$
    steps and return to the pool together. The support is bounded, $D_v \le d$ for all $v$.
  \item[\textnormal{(A3) All-or-nothing blocking (M/G/R/R).}] At a type-$v$ arrival, if at
    least $k$ servers are free the policy assigns a set $\mathcal{A}_t$ of exactly $k$
    of them, yielding reward $\sum_{r \in \mathcal{A}_t} s_{v_t r}$; if fewer than $k$ are
    free the task is dropped and leaves the system.
  \item[\textnormal{(A4) Work-conserving policies.}] We consider only work-conserving
    policies: Whenever at least $k$ servers are free at a type-$v$ arrival, the policy
    assigns $k$ of them. A policy chooses which servers to assign, but may not leave a
    task unserved when it could be served.
\end{description}

The objective is to maximize long-run average reward per step:
\begin{equation}
    \operatorname{maximize} \quad
    \lim_{T \to \infty} \frac{1}{T} \sum_{t=1}^{T} \sum_{r \in \mathcal{A}_t} s_{v_t r}.
    \label{eq:objective}
\end{equation}

\begin{figure}[t]
\centering
\begin{tikzpicture}[>=Stealth, scale=0.9]
  \node[font=\small] at (-4.5, 1.5) {Tasks};
  \node at (-4.5, 0.8) {$\cdots$};
  \draw[->] (-4.3, 0.8) -- (-4.1, 0.8);
  \draw[thick] (-1.9, 0.5) rectangle (-0.9, 1.1);
  \node[font=\small] at (-1.4, 0.8) {$v_t$};
  \node[font=\small, right] at (-0.7, 0.8) {type $v_t \sim (\lambda_v)$};
  \node[font=\small] at (-4.5, -0.8) {Servers};
  \draw[thick] (-3, -0.8) circle (0.35);
  \node[font=\footnotesize] at (-3, -0.8) {$1$};
  \draw[pattern=north east lines] (-1.5, -0.8) circle (0.35);
  \node[font=\footnotesize] at (-1.5, -0.8) {$2$};
  \draw[thick] (0, -0.8) circle (0.35);
  \node[font=\footnotesize] at (0, -0.8) {$3$};
  \draw[pattern=north east lines] (1.5, -0.8) circle (0.35);
  \node[font=\footnotesize] at (1.5, -0.8) {$4$};
  \draw[thick] (3, -0.8) circle (0.35);
  \node[font=\footnotesize] at (3, -0.8) {$5$};
  \node at (4.2, -0.8) {$\cdots$};
  \node[font=\scriptsize] at (-4.5, -1.5) {$x_i$:};
  \node[font=\scriptsize] at (-3, -1.5) {$0$};
  \node[font=\scriptsize] at (-1.5, -1.5) {$2$};
  \node[font=\scriptsize] at (0, -1.5) {$0$};
  \node[font=\scriptsize] at (1.5, -1.5) {$1$};
  \node[font=\scriptsize] at (3, -1.5) {$0$};
  \node[font=\scriptsize] at (-4.5, -2.1) {$s_{v_t}$:};
  \node[font=\scriptsize] at (-3, -2.1) {$s_{v_t,1}$};
  \node[font=\scriptsize, gray] at (-1.5, -2.1) {--};
  \node[font=\scriptsize] at (0, -2.1) {$s_{v_t,3}$};
  \node[font=\scriptsize, gray] at (1.5, -2.1) {--};
  \node[font=\scriptsize] at (3, -2.1) {$s_{v_t,5}$};
  \draw[decorate, decoration={brace, amplitude=5pt, mirror}]
    (-3.5, -2.5) -- (4.5, -2.5) node[midway, below=6pt, font=\small] {$R$ server types};
  \draw[->] (-1.4, 0.45) -- (-3, -0.4);
  \draw[->] (-1.4, 0.45) -- (0, -0.4);
  \node[font=\scriptsize] at (0.8, 0.1) {assign $k = 2$ servers};
\end{tikzpicture}
\caption{Online assignment with reusable servers. At each step $t$, a task of type $v_t$ arrives i.i.d.\ with $\Pr[v_t = v] = \lambda_v$ and is assigned to $k$ available servers. State $x_i$ denotes steps until server $i$ is free; $x_i = 0$ means available. Scores $s_{v_t, r}$ are shown for available servers only; hatched circles are busy.}
\label{fig:model}
\end{figure}
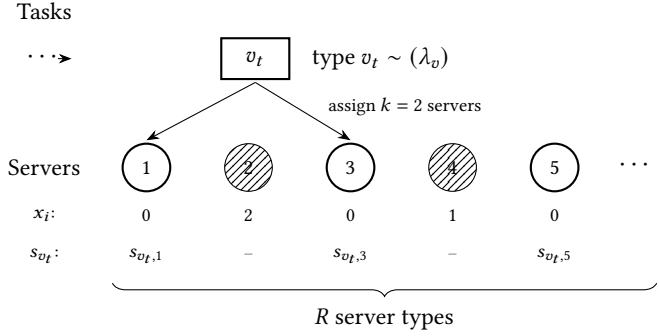

\subsection{Why This Problem Is Hard}
\label{subsec:hard}

A natural baseline for our problem is a greedy assignment policy: Upon arrival of a type-$v$ task, assign the $k$ currently available servers with the highest scores for type $v$. Greedy assignment fails to account for future demand: A server that scores well across many task types will be claimed by nearly every arrival and be unavailable when a task arrives for which it is the best match. More precisely, the value of assigning server $r$ to the current task depends not only on the immediate reward $s_{vr}$, but on the expected reward forgone by future tasks that arrive while $r$ is occupied. Despite its limitations, greedy is the dominant policy in practice. Automated reviewer assignment systems rank available reviewers by expertise score and assign the top $k$~\cite{ZHAO2022103028}. In call centers, the routing component of the Wallace and Whitt~\cite{wallace2005staffing} staffing policy dispatches each call to the best-available agent for its type.

Optimal control can be formulated as an MDP over the joint availability state of all servers, but the state space is exponential in the number of servers, making exact dynamic programming intractable at any practical scale. We instead formulate our problem using a restless bandit framework and derive an occupancy-measure based LP, OCC-LP. The dual of OCC-LP decomposes the problem into independent per-server subproblems, yielding a closed-form index policy.

\subsection{Our Contributions}
\label{subsec:contributions}

We develop \Lag, an index policy for online assignment of reusable, quality-differentiated resources, and show it is asymptotically optimal in theory and performs well in practice. \Lag\ is asymptotically optimal in a mean-field limit, and on real peer-review traces and synthetic instances it attains 83--99\% of a tight LP upper bound, including well outside the regime where our optimality theory applies, and improves over the natural greedy policy by up to 12.3\% on peer-review datasets. We obtain \Lag\ by viewing the problem through a restless-bandit lens:
Each server is an arm and the arriving task type is a context, making the system a contextual restless multi-armed bandit (CRMAB) to which we can adapt restless-bandit index machinery~\cite{whittle1988restless,weber1990index,verloop2016asymptotic}.

\paragraph{A tight LP bound and the \Lag\ policy.}
LP relaxations are a standard tool for upper-bounding the best achievable reward in online matching and restless bandit problems. Our occupancy-measure LP (OCC-LP) does this for our assignment problem. Our OCC-LP is a strictly tighter variant of the KIID LP of Dickerson et al.~\cite{dickerson2021allocation}. The dual of OCC-LP decouples the assignment problem across servers, and each per-server subproblem is solved by a threshold policy. This yields \Lag: A scalar index $L(v,r)$ per (task type, server) pair, computed once from the LP solution, after which each arrival is assigned to its highest-index available servers. Greedy ranks servers by score alone. \Lag\ subtracts from each score the value of keeping that server free for future tasks, so it avoids spending a server that later arrivals need more.

\paragraph{A new mean-field scaling regime.}
Proving optimality for a \emph{contextual} restless bandit requires a scaling that existing analyses do not provide: The classical $n$-only scaling of Weber and Weiss~\cite{weber1990index} keeps one random arrival per step, so the fluid system never settles to a deterministic fixed point. We introduce an $(m,n)$-scaling in which $n$ copies of each server and $m$ i.i.d.\ arrivals grow together; scaling $m$ concentrates the empirical context distribution each step, restoring the deterministic fluid limit needed for a fixed-point analysis. In this regime we prove \Lag\ is asymptotically optimal:

\begin{theorem}[Informal; see Theorem~\ref{thm:asymptotic_optimality}]
\label{thm:informal}
In the $(m,n)$-scaled system, under a mild non-degeneracy condition,
\[
  \frac{\Reward_{m,n}(\Lag)}{mn} \;\to\; \mathrm{OPT}(\textup{OCC-LP}(0)) \qquad \text{as } m,n \to \infty.
\]
\end{theorem}

\paragraph{Empirical performance.}
On real peer-review datasets and structured synthetic instances, \Lag\ attains 83--99\% of the OCC-LP upper bound across all load levels, far outside the mean-field regime, and improves over greedy assignment by up to 12.3\% on peer-review datasets and up to 36\% on synthetic instances. We also validate convergence to the mean-field limit and characterize how the gap between \Lag\ and greedy grows with score heterogeneity.

\section{Related Work}
\label{sec:related}

Two research communities have studied problems related to ours. The queueing literature on parallel server systems (Section~\ref{sec:parallel}) asks how to route and staff heterogeneous servers efficiently; the online matching literature (Section~\ref{sec:online-assignment}) studies assignment of reusable resources to arriving tasks and evaluates algorithms by competitive ratio. We discuss how each community has approached these problems and where our model and results depart from theirs. We also discuss the restless multi-armed bandit literature (Section~\ref{sec:rmab}), which we draw on for both the CRMAB formulation of our problem and the fluid-scaling proof techniques behind our main result.

\subsection{Parallel Server Systems}
\label{sec:parallel}

The parallel service literature studies systems in which heterogeneous servers must be matched to a stream of arriving tasks subject to compatibility constraints. Call centers, where agents with different skills are assigned to different call types, are the canonical motivating application. We highlight two threads most relevant to our work: Routing policy design and staffing.

\paragraph{Routing.}
Several papers ask which routing policy optimizes system performance, typically measured by throughput or delay. Adan and Weiss~\cite{adan2014skill} analyze the FCFS-ALIS policy under binary compatibility and derive exact product-form stationary distributions for the joint occupancy state. Armony~\cite{armony2005dynamic} studies a single customer class routed to server pools of differing speeds and shows that Fastest-Servers-First asymptotically minimizes queue length in the Halfin-Whitt regime. Mandelbaum and Stolyar~\cite{mandelbaum2004scheduling} consider flexible multi-skilled servers with convex delay costs; their service rates $\mu_{ij}$ vary by both customer type and server. They show that the generalized $c\mu$-rule (which routes each customer to the server maximizing the instantaneous rate of cost decrease) is asymptotically optimal in heavy traffic. Atar~\cite{atar2005scheduling} proves related results for multi-class systems via a diffusion control approach. In all of these, service heterogeneity enters as a capacity or cost parameter; the optimization objective is throughput or delay, not match quality.

\paragraph{Staffing.}
Wallace and Whitt~\cite{wallace2005staffing} ask how many agents of each skill type a call center should employ to meet per-class service targets. Their three-stage algorithm combines Erlang-based total headcount, square-root staffing for primary skills, and fair assignment of secondary skills, and requires that service times not depend on the call type or the agent.

All of the above treat compatible assignments as interchangeable in quality. We replace the binary compatibility graph with a real-valued score matrix $S \in \mathbb{R}^{V \times R}$, where $s_{vr}$ is the reward for assigning server type $r$ to a type-$v$ task, and optimize long-run average reward rather than throughput or delay. The OCC-LP is the natural analogue of the throughput LP used throughout this literature; \Lag\ is the index policy derived from its dual.

\subsection{Online Assignment with Reusable Resources}
\label{sec:online-assignment}

Online bipartite matching with reusable resources is a well-studied problem~\cite{huang2024online}: Agents matched to an arriving task become temporarily unavailable, returning to the pool after a random occupation time, capturing ridesharing, crowdsourcing, organ donation, and server allocation. The dominant performance criterion is competitive ratio, an algorithm's expected reward guaranteed as a fraction of the expected reward of an offline oracle that knows all future arrivals. Dickerson et al.~\cite{dickerson2021allocation} study the OM-RR-KAD model (Online Matching with Reusable Resources under a Known Adversarial Distribution) and derive an LP-based algorithm (ADAP) achieving competitive ratio $\tfrac{1}{2} - \varepsilon$ for any $\varepsilon > 0$; they also establish that no non-adaptive algorithm can achieve better than $\tfrac{1}{2} + o(1)$. Sumita et al.~\cite{sumita2022online} extend this to multiple agents per task and agent rejections. Shanks, Yu and Jacobson~\cite{shanks2023approximation} study a variant with rank-1 similarity scores; Delong et al.~\cite{delong2023online} achieve $\approx 0.589$ in the vertex-weighted adversarial case.

Rather than targeting a fractional guarantee relative to an offline oracle, our work directly maximizes expected long-run average reward in the stationary i.i.d.\ setting: The OCC-LP tightly characterizes the achievable long-run average reward rate, and \Lag\ is asymptotically optimal with respect to this benchmark in the mean-field limit. This approach achieves reward approaching 100\% of the OCC-LP bound as the system scales, compared to the $\tfrac{1}{2}$ ceiling that bounds competitive ratio performance in this literature. Empirically, \Lag\ consistently achieves 83--99\% of the OCC-LP bound on application traces, well above what competitive ratio algorithms achieve in practice. Our OCC-LP is related to, but strictly tighter than, the benchmark LP of~\cite{dickerson2021allocation}. The algorithm of~\cite{dickerson2021allocation} achieving the competitive ratio guarantee requires running simulation to precompute its assignment probabilities and is not practical to deploy; we include their ALG-SC-LP heuristic as an experimental baseline in Section~\ref{sec:experiments} and show that \Lag\ consistently outperforms it.

\subsection{Restless Bandits and Mean-Field Analysis}
\label{sec:rmab}

A restless multi-armed bandit (RMAB)~\cite{whittle1988restless} consists of $n$ arms, each an independent Markov chain. At each step, the decision-maker selects exactly $k$ arms to \emph{activate} (the per-step budget); each arm receives a reward and transitions according to its active or passive dynamics, and the objective is to maximize long-run average reward subject to the per-step budget constraint. Optimal RMAB control is PSPACE-hard~\cite{papadimitriou1987complexity}, which has motivated a rich literature on tractable index heuristics with theoretical guarantees~\cite{nino2023markovian}.

Whittle~\cite{whittle1988restless} proposed relaxing the per-step budget to a time-average constraint via Lagrangian duality, which decouples the arms into $n$ independent single-arm subproblems. When each arm is \emph{indexable} (meaning the set of states preferring passivity grows monotonically in the Lagrange multiplier), the break-even multiplier defines a scalar Whittle index for each (arm, state) pair, and the resulting policy activates the $k$ arms with the highest current indices. Weber and Weiss~\cite{weber1990index} proved this policy asymptotically optimal as $k, n \to \infty$ with the activation fraction $k/n$ fixed, under the assumption that the fluid ODE describing the mean-field dynamics has a unique globally stable equilibrium. Their proof establishes that the stochastic occupancy measure concentrates on this equilibrium as $n \to \infty$, and that the equilibrium coincides with the LP-relaxation optimum. Verloop~\cite{verloop2016asymptotic} substantially extended these results to multi-class RMAB problems with heterogeneous arms, dynamic arm arrivals and departures, and multiple simultaneous actions, using a combination of fluid-scaling techniques and LP analysis; our proof of Theorem~\ref{thm:asymptotic_optimality} follows Verloop's approach directly. Avrachenkov, Borkar and Shah~\cite{avrachenkov2026lagrangian} study a Lagrangian index policy for average-reward RMABs and prove asymptotic optimality in the homogeneous-arm setting.

A \emph{contextual} RMAB (CRMAB) extends this framework: At each step a context is observed that modulates arm rewards and, possibly, arm transitions. Our problem is naturally a CRMAB: Each server type $r$ is a bandit arm and the arriving task type $v_t$ is the context, determining both the reward and the occupation time for each arm. Chen and Hou~\cite{chen2024contextual} introduce a CRMAB in which the context follows a Markov process, arm transitions depend on the context, and the per-step budget is context-dependent; they derive a Lagrangian index policy via dual decomposition, but demonstrate asymptotic optimality only numerically. Our work proves asymptotic optimality of \Lag\ in a new $(m,n)$-scaling regime, in which both the number of simultaneous arrivals $m$ and the number of resource copies $n$ grow (Section~\ref{sec:mf}). This regime is necessary: In the standard $n \to \infty$ scaling with one arrival per step, the per-step context is random, which prevents the fluid system from converging to a deterministic fixed point, and the Lagrangian reward may not reach the fluid LP bound. Scaling $m$ alongside $n$ allows the empirical context distribution to concentrate at each step by the law of large numbers, making the fluid limit deterministic and amenable to the fixed-point analysis of~\cite{verloop2016asymptotic}.

\section{Our Proposed Policy}
\label{sec:policy}

We model the problem as a contextual restless multi-armed bandit (CRMAB). 

\begin{description}
    \item[Arm.] Each server $r \in [R]$ is an arm. At step $t$, arm $r$ has state
        $X_r(t) \in \{0, 1, \ldots, d\}$, where $X_r(t) = 0$ means $r$ is free and
        $X_r(t) = j > 0$ means $r$ has $j$ steps of service remaining; here $d$
        bounds the service duration ($D_v \le d$ for all $v$).

    \item[Context.] At each step $t$, a task of type $v_t \in [V]$ arrives i.i.d.\
        with $\Pr[v_t = v] = \lambda_v$, and no task arrives with probability
        $1 - \sum_{v} \lambda_v$, independently of the arm states.

    \item[Action.] When at least $k$ servers are free, a policy selects an assignment
        set $\mathcal{A}_t \subseteq \{r \in [R] : X_r(t) = 0\}$ with
        $|\mathcal{A}_t| = k$; when fewer than $k$ are free the task is dropped
        ($\mathcal{A}_t = \emptyset$). We consider only work-conserving policies, which
        assign whenever $k$ servers are free.

    \item[Transitions.] An assigned server becomes busy for the task's service
        duration, and every other busy server has one step of service elapse. That is,
        for each $r \in [R]$,
        \[
        X_r(t+1) =
        \begin{cases}
            D_{v_t} & r \in \mathcal{A}_t, \\
            \bigl(X_r(t) - 1\bigr)^+ & r \notin \mathcal{A}_t,
        \end{cases}
        \]
        where $D_{v_t}$ is a single duration drawn at the moment of assignment and
        shared by all $k$ assigned servers, so they free together.

    \item[Reward.] Assigning $\mathcal{A}_t$ to a type-$v_t$ task yields reward
        $\sum_{r \in \mathcal{A}_t} s_{v_t r}$. The objective is to
        maximize~\eqref{eq:objective}.
\end{description}

\begin{table}[t]
\centering
\caption{Summary of notation used in Section~\ref{sec:policy}.}
\label{tab:notation}
\small
\begin{tabular}{@{}ll@{}}
\toprule
\textbf{Symbol} & \textbf{Description} \\
\midrule
\multicolumn{2}{@{}l}{\textit{Problem parameters}} \\[2pt]
$R$              & Number of servers \\
$V$              & Number of task types \\
$\lambda_v \geq 0$     & Arrival rate of a type-$v$ task (probability per step; $\sum_v \lambda_v \le 1$) \\
$k$            & Number of servers assigned per type-$v$ task \\
$s_{vr} \in \mathbb{R}$ & Match reward for assigning server $r$ to a type-$v$ task \\
$D_v \le d$  & Service duration of a type-$v$ task, shared by its $k$ servers (random integer) \\
\midrule
\multicolumn{2}{@{}l}{\textit{CRMAB model}} \\[2pt]
$X_r(t) \in \{0,\ldots,d\}$ & State of server $r$ at step $t$;\; $X_r(t)=0$ means free \\
$\mathcal{A}_t \subseteq [R]$ & Set of $k$ servers assigned at step $t$ \\
\midrule
\multicolumn{2}{@{}l}{\textit{OCC-LP variables}} \\[2pt]
$u_{vr} \geq 0$  & Long-run assignment rate of server $r$ to type-$v$ tasks (per step) \\
$\phi_r \geq 0$  & Steady-state free fraction of server $r$ \\
\midrule
\multicolumn{2}{@{}l}{\textit{Lagrangian index}} \\[2pt]
$\ell_v^*$       & Dual variable for the type-$v$ demand constraint~\eqref{lp:demand} \\
$q_r^*$          & Dual variable (Arm-$r$ LP value) for the capacity constraint~\eqref{lp:capacity} \\
$L(v,r) = s_{vr} - \ell_v^* - \mathbb{E}[D_v]\,q_r^*$ & Lagrangian index of server $r$ for task type $v$ \\
\bottomrule
\end{tabular}
\end{table}

\subsection{OCC-LP: The Occupancy-Measure LP Upper Bound}\label{sec:upper-bound-LP}

The \emph{occupancy-measure LP} (OCC-LP) is a linear program over the long-run
assignment rates of a work-conserving policy; its optimum upper-bounds the reward of
any such policy, as we prove in Proposition~\ref{prop:lp_upper_bound}. The decision
variables are $u_{vr} \geq 0$ for $v \in [V]$, $r \in [R]$, the long-run rate of
assignments of server $r$ to type-$v$ tasks (assignments per step), and
$\phi_r \geq 0$ for $r \in [R]$, the steady-state free fraction of server $r$. One
further quantity enters the LP as a constant: The \emph{blocking probability} $B$, the
long-run fraction of arrivals dropped because fewer than $k$ servers are free. By
Lemma~\ref{lem:blocking}, $B$ is the same for every work-conserving policy and every task
type.
\begin{alignat}{3}
\max_{u,\,\phi} \quad
  & \sum_{v \in [V],\, r \in [R]} s_{vr}\, u_{vr}
  & & \tag{OCC-LP}\label{OCC-LP} \\
\text{s.t.} \quad
  & \sum_{r \in [R]} u_{vr} = k \lambda_v\,(1-B)
  & \quad & \forall\, v \in [V] \label{lp:demand} \\
  & \phi_r = 1 - \sum_{v \in [V]} \mathbb{E}[D_v]\, u_{vr}
  & & \forall\, r \in [R] \label{lp:capacity} \\
  & 0 \leq u_{vr} \leq \lambda_v \phi_r
  & & \forall\, v \in [V],\; r \in [R] \label{lp:bound}
\end{alignat}
Constraint~\eqref{lp:demand} fixes the type-$v$ assignment rate: Type-$v$ tasks arrive at
rate $\lambda_v$, a fraction $1-B$ of them are served, and each served task occupies
exactly $k$ servers, giving rate $k\lambda_v(1-B)$.
We call $k\lambda_v(1-B)$ the demand for type $v$.
Constraint~\eqref{lp:capacity} is Little's Law: Each assignment of type $v$ to
server $r$ occupies it for expected duration $\mathbb{E}[D_v]$ at rate
$u_{vr}$, so the total busy fraction is $\sum_{v} \mathbb{E}[D_v]\,u_{vr}$.
Constraint~\eqref{lp:bound} ensures the assignment rate $u_{vr}$ does not exceed
the rate at which server $r$ is simultaneously free and a type-$v$ task arrives;
by independence of arrivals and server state in the fluid limit, this rate is
$\lambda_v \phi_r$.

\begin{lemma}\label{lem:blocking}
Let $N_t$ denote the number of busy servers at step $t$. Under any work-conserving
policy the process $(N_t)$ undergoes the same state transitions, and the steady-state
blocking probability $B := \lim_{t\to\infty}\Pr[N_t \ge R-k+1]$ is independent of both
the policy and the arriving task type.
\end{lemma}
\begin{proof}
Run any two work-conserving policies on the same arrival sequence and the same service time 
draws, starting from $N_0 = 0$. Suppose $N_t$ agrees under both. A work-conserving policy
assigns exactly $k$ servers iff $N_t \le R-k$, so under the common value of $N_t$ both
policies make the same accept-or-drop decision. On acceptance both occupy $k$ servers for
the same duration $D_{v_t}$, which depends only on the type $v_t$ and not on which servers
are chosen, so the same number become busy and the same number later free; hence $N_{t+1}$
again agrees. By induction the two processes coincide step by step, so they share the same
transitions and the same steady-state blocking probability $B$. For type-independence,
$v_t$ is i.i.d.\ and independent of the history, hence of $N_t$, while the event
$\{N_t \ge R-k+1\}$ is a function of $N_t$ alone (as $k$ is constant); thus
$\Pr[\text{block}\mid v_t = v] = \Pr[N_t \ge R-k+1]$ for every $v$.
\end{proof}

\begin{remark}\label{rem:erlang}
In continuous time the blocking probability has a closed form. Modeling each task as a
call that holds $k$ of $R$ circuits for one duration $D_v$ gives an Erlang loss
network~\cite{kelly1991loss}, whose blocking depends only on the means:
$B = \mathrm{ErlangB}\!\left(R/k,\; \sum_{v} \lambda_v \mathbb{E}[D_v]\right)$.
Our discrete-time model admits no such closed form: $B$ is sensitive to the service
distribution beyond its mean. We therefore compute $B$ numerically; the Erlang expression
still serves as a useful reference value.
\end{remark}

We write $\text{OCC-LP}(B)$ to make the blocking probability explicit, and
$\text{OCC-LP}(0)$ for the $B=0$ instance in which the demand constraint reads
$\sum_r u_{vr} = k \lambda_v$. As the system scales the blocking vanishes
(Section~\ref{sec:mf}), so $\text{OCC-LP}(0)$ is the fluid limit from which we derive
\Lag\ and against which we prove its optimality, while $\text{OCC-LP}(B)$ is the valid
upper bound at finite scale.

\begin{proposition}[LP upper bound]\label{prop:lp_upper_bound}
For any work-conserving policy $\pi$, with $B$ the blocking probability of
Lemma~\ref{lem:blocking},
\[
  \lim_{T\to\infty}\frac{1}{T}\sum_{t=1}^{T}\sum_{r\in\mathcal{A}_t} s_{v_t r}
  \;\leq\; \mathrm{OPT}(\text{\upshape OCC-LP}(B)).
\]
\end{proposition}
\begin{proof}
Fix a policy $\pi$. Let $\phi_r^\pi$ be the long-run fraction of steps
server $r$ is free, and $u_{vr}^\pi$ the long-run rate of type-$v$
assignments to server $r$:
\[
  \phi_r^\pi := \lim_{T\to\infty} \frac{1}{T}\sum_{t=1}^T \mathbf{1}[X_r^\pi(t)=0],
  \qquad
  u_{vr}^\pi := \lim_{T\to\infty} \frac{1}{T}\sum_{t=1}^{T}
    \mathbf{1}[v_t = v,\; r \in \mathcal{A}_t].
\]
For policy $\pi$, the objective equals $\sum_{v,r} s_{vr}\,u_{vr}^\pi$.
We show $(u^\pi,\phi^\pi)$ is feasible for \eqref{OCC-LP}.

\emph{Constraint~\eqref{lp:demand}.}
At each type-$v$ arrival a work-conserving policy assigns exactly $k$ servers when at
least $k$ are free, and drops the task otherwise. Let $T(v) := \#\{t \leq T : v_t = v\}$
and let $T_{\mathrm{acc}}(v)$ count the accepted type-$v$ arrivals. Then
\[
  \sum_{r\in[R]} u_{vr}^\pi
  = \lim_{T\to\infty}\frac{k\,T_{\mathrm{acc}}(v)}{T}
  = k\,\lambda_v\,(1-B),
\]
using $T(v)/T\to \lambda_v$ and $T_{\mathrm{acc}}(v)/T(v)\to 1-B$ a.s., where $B$ is the
policy- and type-independent blocking probability of Lemma~\ref{lem:blocking}.

\emph{Constraint~\eqref{lp:capacity}.}
By Little's Law applied to server $r$: The long-run assignment rate to
server $r$ is $\sum_v u_{vr}^\pi$, the mean service time per assignment
is $\sum_v \mathbb{E}[D_v]\,u_{vr}^\pi \big/ \sum_v u_{vr}^\pi$,
and the average occupancy is $1-\phi_r^\pi$.  Hence
$1 - \phi_r^\pi = \sum_v \mathbb{E}[D_v]\,u_{vr}^\pi$.

\emph{Constraint~\eqref{lp:bound}.}
Server $r$ can only be assigned when free, so
\[
  u_{vr}^\pi \;\leq\;
  \lim_{T\to\infty}\frac{1}{T}\sum_{t=1}^{T}
    \mathbf{1}[v_t=v,\; X_r^\pi(t)=0].
\]
Let $\mathcal{F}_t$ denote the history before step $t$ (all arrivals,
assignments, and service draws through step $t-1$). The state $X_r^\pi(t)$ is
$\mathcal{F}_t$-measurable, while $v_t$ is i.i.d.\ with $\Pr[v_t=v]=\lambda_v$ and
independent of $\mathcal{F}_t$, so
$M_T = \sum_{t=1}^{T}(\mathbf{1}[v_t=v]-\lambda_v)\,\mathbf{1}[X_r^\pi(t)=0]$
is a martingale with bounded increments, giving $M_T/T\to 0$ a.s.  Hence
\[
  \lim_{T\to\infty}\frac{1}{T}\sum_{t=1}^{T}
    \mathbf{1}[v_t=v,\;X_r^\pi(t)=0]
  = \lambda_v\,\lim_{T\to\infty}\frac{1}{T}\sum_{t=1}^{T}\mathbf{1}[X_r^\pi(t)=0]
  = \lambda_v\,\phi_r^\pi.
\]

Therefore $(u^\pi,\phi^\pi)$ is feasible for $\text{OCC-LP}(B)$, giving
$\sum_{v,r}s_{vr}\,u_{vr}^\pi \leq \mathrm{OPT}(\text{OCC-LP}(B))$.
\end{proof}

Note that this upper bound is similar to, but tighter than, the KIID variant of \cite[LP (1)]{dickerson2021allocation}.

\subsection{LAG, the Lagrangian Index Policy}\label{sec:index_policy}

We derive a Lagrangian index policy via dual decomposition of the OCC-LP\@. We show in Section~\ref{sec:experiments} that it achieves 83--99\% of the OCC-LP upper bound across all evaluated instances and loads, and prove asymptotic optimality in Section~\ref{sec:mf}.

\paragraph{Lagrangian decomposition.}
We derive the index from the fluid limit $\text{OCC-LP}(0)$. Specifically, we form its Lagrangian by
relaxing the equality demand constraint~\eqref{lp:demand} (with $B=0$) using multipliers
$\ell_v \in \mathbb{R}$, one per task type $v \in [V]$:
\begin{align*}
\mathcal{L}(u, \ell)
  &= \sum_{v \in [V],\, r \in [R]} s_{vr}\, u_{vr}
     + \sum_{v \in [V]} \ell_v\!\left(k \lambda_v - \sum_{r \in [R]} u_{vr}\right) \\
  &= \sum_{r \in [R]} \sum_{v \in [V]} (s_{vr} - \ell_v)\, u_{vr}
     + \sum_{v \in [V]} \ell_v\, k \lambda_v,
\end{align*}
subject to \eqref{lp:capacity} and \eqref{lp:bound}. For fixed $\ell$, the maximization over $u$ decouples across servers $r \in [R]$. For each $r \in [R]$, the Arm-$r$ LP is:
\begin{align}
\max_{(u_{vr})_{v \in [V]}} \quad
  & \sum_{v \in [V]} (s_{vr} - \ell_v)\, u_{vr}
  \tag{Arm-$r$ LP}\label{eq:per_server_lp} \\
\text{s.t.} \quad
  & 0 \leq u_{vr} \leq \lambda_v \phi_r \qquad \forall\, v \in [V] \notag \\
  & \phi_r = 1 - \sum_{v \in [V]} \mathbb{E}[D_v]\, u_{vr}. \notag
\end{align}
Let $q_r(\ell)$ denote the optimal value of the Arm-$r$ LP.

\paragraph{Dual LP.}
The Lagrangian dual problem is:
\[
  \min_{\ell \in \mathbb{R}^V}\; \sum_{r \in [R]} q_r(\ell) + \sum_{v \in [V]} \ell_v\, k \lambda_v.
\]
By strong duality, the optimal dual value equals $\mathrm{OPT}(\text{OCC-LP}(0))$.
Let $\ell^* = (\ell_v^*)_{v \in [V]}$ be an optimal dual solution,
$(u^*, \phi^*)$ the corresponding primal solution,
and $q_r^* := q_r(\ell^*)$.

\begin{revblock}
\begin{definition}[$\Lag$ Policy]\label{def:lag}
Assume $\lambda_v>0$ for every $v\in[V]$. Fix an optimal dual solution $\ell^*$ and an optimal solution $(u^*,\phi^*)$ of $\text{OCC-LP}(0)$. For $v\in[V]$ and $r\in[R]$, define the \emph{Lagrangian index}
\begin{equation}\label{eq:index}
  L(v, r) \;:=\; s_{vr} - \ell_v^* - \mathbb{E}[D_v]\, q_r^*
\end{equation}
and the \emph{tie weight} $w_{vr}:=u^*_{vr}/\lambda_v$. {\color{red}When a task of type $v$ arrives, \Lag\ drops it if fewer than $k$ servers are free; otherwise, it assigns $k$ free servers in decreasing index order, with zero-index ties calibrated to the LP assignment rates. The precise allocation and rounding rules are specified in Appendix~\ref{app:lag-ties} and are part of the policy definition.}
\end{definition}
\end{revblock}

In $L(v,r)$, the score $s_{vr}$ is the immediate reward. The price $\ell_v^*$ is the opportunity cost of serving a type-$v$ task. The term $\mathbb{E}[D_v]\,q_r^*$ is the cost of keeping server $r$ busy, since each step of its time is worth $q_r^*$.

{\color{red}The Lagrangian relaxation replaces the per-arrival assignment requirement by time-average demand constraints. For the selected primal optimum $(u^*,\phi^*)$, consider the decoupled policy that, when server $r$ is free and type $v$ arrives, accepts with probability
\[
  \alpha_{vr}=\frac{u^*_{vr}}{\lambda_v\phi_r^*}.
\]
Feasibility gives $\phi_r^*>0$ and $\alpha_{vr}\in[0,1]$. The sign characterization below implies $\alpha_{vr}=1$ for positive indices and $\alpha_{vr}=0$ for negative indices; zero-index acceptance is calibrated by this ratio. The construction in Lemma~\ref{lem:bandit-lp-equivalence} realizes the selected occupancies, so the total original reward is $\sum_{v,r}s_{vr}u^*_{vr}=\mathrm{OPT}(\text{OCC-LP}(0))$. This decoupled policy meets demand in time average but need not assign exactly $k$ servers to each arrival.}

\begin{proposition}\label{prop:lag-characterization}
There exists an optimal solution for OCC-LP, $(u^*, \phi^*)$, which satisfies, for every $v \in [V]$, $r \in [R]$:
\begin{align*}
  L(v,r) > 0 &\implies u^*_{vr} = \lambda_v \phi_r^*, \\
  L(v,r) = 0 &\implies u^*_{vr} \in [0,\, \lambda_v \phi_r^*], \\
  L(v,r) < 0 &\implies u^*_{vr} = 0.
\end{align*}
\end{proposition}

In the full model, LAG assigns exactly $k$ servers to each arrival, so this decoupled optimality result does not translate directly. In Section~\ref{sec:mf}, we establish a mean-field scaling regime in which this distinction vanishes.

To prove Proposition~\ref{prop:lag-characterization}, we translate the Arm-$r$ LP into a bandit problem.
Lemma~\ref{lem:bandit-lp-equivalence} shows the optimal average reward of the bandit equals~$q_r^*$.
{\color{red}Lemma~\ref{lem:bandit-optimal} shows that, at the dual-optimal penalties, positive indices require acceptance, negative indices require rejection, and zero indices permit either action or randomization.}
{\color{red}Lemma~\ref{lem:cs} then establishes Proposition~\ref{prop:lag-characterization} by expressing the primal--dual gap as a sum of nonnegative terms.} Proofs of Lemmas~\ref{lem:bandit-lp-equivalence} and \ref{lem:bandit-optimal} are deferred to Appendix~\ref{app:lag-derivation}.

\begin{definition}[Arm-$r$ bandit]\label{def:bandit}
Given a penalty vector $\boldsymbol{\ell} \in \mathbb{R}^V$, the \emph{Arm-$r$ bandit} is a
single server observed in discrete time:
\begin{description}
  \item[State.] The server is either free ($X = 0$) or busy with
    $X \in \{1,\ldots,d\}$ steps of service remaining.
    At each step a type-$v$ task arrives with probability $\lambda_v$ (and none with
    probability $1 - \sum_v \lambda_v$), independently of the server state.

  \item[Action.] When the server is free and a type-$v$ task arrives, a policy
    chooses to accept ($a=1$) or reject ($a=0$).
    Acceptance is only possible when the server is free.

  \item[Reward.] Accepting a type-$v$ task yields immediate reward $s_{vr} - \ell_v$;
    the server then becomes busy for $D_v$ steps, drawn i.i.d.\ ($D_v \le d$)
    and independently of all other quantities.

  \item[Objective.] Let $g_\pi$ denote the long-run average reward per step under
    policy~$\pi$ started from the free state. The time-average Arm-$r$ bandit
    maximizes $g_\pi$ over policies~$\pi$.
\end{description}
The penalty $\ell_v^*$ is the shadow price of the demand constraint for
type~$v$: It reflects the opportunity cost of serving a type-$v$ arrival,
as determined by the global demand constraints~\eqref{lp:demand}.
\end{definition}

\begin{lemma}[Arm-$r$ Bandit--LP equivalence]\label{lem:bandit-lp-equivalence}
{\color{red}At penalties $\ell^*$, the optimal average reward of the Arm-$r$ bandit equals the value of the Arm-$r$ LP:}
$g_* = q_r^*$.
\end{lemma}

\begin{lemma}\label{lem:bandit-optimal}
{\color{red}At penalties $\ell^*$, an optimal stationary policy for the time-average Arm-$r$ bandit can be specified by acceptance probabilities $\pi^*(0,v)$ satisfying:}
\begin{align*}
  s_{vr} - {\color{red}\ell_v^*} > \mathbb{E}[D_v]\, q_r^* &\implies \pi^*(0, v) = 1, \\
  s_{vr} - {\color{red}\ell_v^*} = \mathbb{E}[D_v]\, q_r^* &\implies \pi^*(0, v) \in {\color{red}[0,1]}, \\
  s_{vr} - {\color{red}\ell_v^*} < \mathbb{E}[D_v]\, q_r^* &\implies \pi^*(0, v) = 0.
\end{align*}
\end{lemma}

{\color{red}An arbitrary optimal policy for the penalized Arm-$r$ bandit need not realize the selected primal optimum: it may choose different actions at zero indices. The calibrated probabilities $\alpha_{vr}$ above realize $(u^*,\phi^*)$ by the achievability construction in Lemma~\ref{lem:bandit-lp-equivalence}. Separately, Lemma~\ref{lem:cs} shows that the sign conditions of Proposition~\ref{prop:lag-characterization} hold for every optimal primal solution paired with any optimal dual solution.}

\section{Mean-Field Optimality}
\label{sec:mf}
In this section, we propose a mean-field scaling regime and prove that $\Lag$ is asymptotically optimal in the scaling regime.

Our analysis follows the Weber--Weiss tradition of mean-field proofs for index policies~\cite{weber1990index,verloop2016asymptotic}, which rests on delicate fixed-point and convergence arguments for finite Markov chains. These arguments require each arm to have a finite state space, which the discrete-time model of Section~\ref{subsec:model} provides: With bounded service durations, each server's state lies in $\{0, \ldots, d\}$.

\subsection{The Mean Field Scaling Regime}

A rich body of prior work studies mean-field asymptotic optimality of index policies for restless bandit problems~\cite{nino2023markovian,weber1990index,verloop2016asymptotic,larranaga2015dynamic}. The classical result of \cite{weber1990index} shows that when there are $n$ identical arms, the Whittle index policy is optimal as $n \to \infty$. However, these results do not directly extend to our \emph{contextual} RMAB setting. It is not immediately clear how a CRMAB should be scaled. We begin by considering two natural approaches:

\paragraph{Batch-scaling ($n$-scaling).}
Suppose we scale the number of resources by a factor of $n$, so that there are $n$ copies of each resource type. At each time step, a single arrival $v_t$ is assigned $nk$ resources out of a pool of size $nR$. This model is depicted in Figure~\ref{fig:n_scaled_system} (Appendix~\ref{app:scaling-figs}). \cite{chen2024contextual} consider a related scaling regime for general CRMABs.

The difficulty is that the arrival process remains stochastic at unit rate. As a result, even as $n \to \infty$, the system does not converge to a deterministic fixed point. Instead, the evolution of the system depends on the realized sequence of contexts, and the fluid limit may exhibit non-trivial dynamics (e.g., limit cycles) driven by the context process.

\paragraph{Buffet-scaling ($m$-scaling).}
Alternatively, suppose there are $m$ i.i.d.\ arrivals at each time step, and $m$ copies of each resource type. Each arrival must be matched to $k$ free resources, thus a total of $mk$ assignments are made at each step. This model is depicted in Figure~\ref{fig:scaling:m} (Appendix~\ref{app:scaling-figs}).

This avoids the limit-cycle issue, since averaging over many arrivals produces a deterministic arrival sequence in the limit. However, this scaling fundamentally changes the problem. At each time step, assignments can be made jointly across all arrivals, allowing resources to effectively choose the most favorable matches among a batch of tasks. This is strictly more powerful than in the original problem, in which each resource only observes a single arrival.

Formally, under this scaling, the induced occupancy measures $(u^{\pi}, \phi^{\pi})$ may violate the constraint \eqref{lp:bound} in \ref{OCC-LP}. The limiting LP (e.g., the KIID LP~\cite{dickerson2021allocation}) therefore provides a looser upper bound than OCC-LP. In our experiments, heuristics based on this LP almost always perform worse than those based on the tight OCC-LP bound, likely because the relaxation models the problem less tightly. This is discussed in Section~\ref{sec:experiments}.

Thus, neither scaling provides the right generalization of the Weber--Weiss regime: $n$-scaling preserves the problem structure but does not yield a deterministic limit, while $m$-scaling yields a deterministic limit but relaxes the problem too much.

\begin{figure}[htb]
\centering
\includegraphics[width=0.6\textwidth]{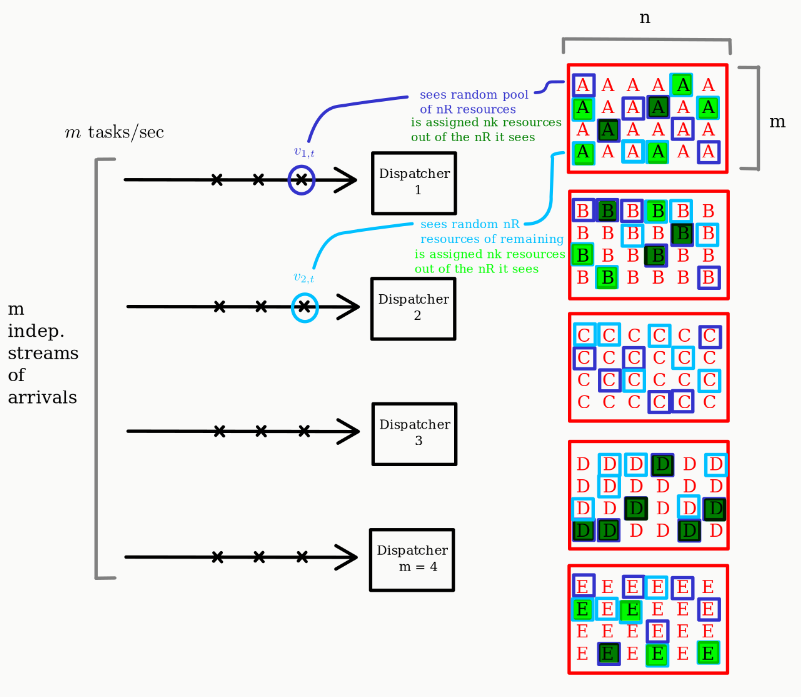}
\caption{$(m,n)$-scaled system ($m{=}4$, $n{=}6$, $R{=}5$, $k{=}2$): $m$ independent streams share
$mn{=}24$ copies of each resource type (A--E), split into $m$ groups of $n$ each per timestep.
Blue cell $=$ Dispatcher$_1$'s partition ($nR{=}30$ total); green cell $=$ assigned this step ($nk{=}12$ total)}
\label{fig:scaling:mn}

\end{figure}

\paragraph{The $(m,n)$-scaling.}
We propose a partitioned scaling that combines both ideas. We scale the number of arrivals and resources, while constructing a system that ensures that each resource only observes one arrival per time step.

At each time step, $m$ i.i.d.\ arrivals occur, and there are $mn$ resources of each type. For each resource type, the $mn$ resources are uniformly partitioned into $m$ groups of size $n$, and each group is assigned to one arrival. Thus, each arrival observes $n$ resources of each type (for a total of $nR$), and must select $nk$ of them. Assigned resources become unavailable for $D_v$ steps (at most $d$).

We scale the arrival rate by $m$, the demand per arrival by $n$, and the resource pool by $mn$, so that the system load remains unchanged:
\[
  \rho = \frac{k}{R}\sum_v \lambda_v\,\E[D_v].
\]

The load $\rho$ is the long-run fraction of busy servers, the same for every
work-conserving policy that serves all demand. Our optimality result requires the system
to stay strictly underloaded.

\begin{assumption}[Strict underload]\label{assum:underload}
\[
  \rho \;<\; 1 - \frac{k}{R}.
\]
\end{assumption}

This is only slightly stronger than $\rho < 1$. For any feasible OCC-LP solution, the bound constraint~\eqref{lp:bound} gives $\sum_r u_{vr} \le \lambda_v \sum_r \phi_r$. The demand constraint~\eqref{lp:demand} fixes the left side at $k\lambda_v$, so $k \le \sum_r \phi_r$. Summing the capacity constraint~\eqref{lp:capacity} over $r$ and using the demand constraint gives $\sum_r \phi_r = R - k\sum_v \lambda_v \mathbb{E}[D_v] = R(1-\rho)$. Hence $k \le R(1-\rho)$, i.e., $\rho \le 1-k/R$. Assumption~\ref{assum:underload} asks for the strict inequality.

Under this scaling, the system admits a mean-field limit as $m,n \to \infty$. The randomness in arrivals and resource states averages out, and each arrival effectively sees a deterministic fraction of available resources of each type. At the same time, the per-resource interaction structure of the original problem is preserved.

The $(m,n)$-scaling admits a natural operational interpretation as a large-scale, distributed matching system with limited visibility, depicted in Figure~\ref{fig:scaling:mn}.

\begin{enumerate}
    \item First, the $m$ arrivals at each time step can be viewed as $m$ independent \emph{dispatchers} (or frontends), each receiving one task. Rather than a single centralized decision-maker with access to the entire resource pool, decisions are made locally and in parallel.
    \item Second, the random partitioning of the $mn$ resources into $m$ groups of size $n$ models \emph{limited access to resources}. Each dispatcher only observes a subset of size $n$ of each resource type, rather than the full system of size $mn$. This reflects practical constraints in large systems, where querying or coordinating across all resources is infeasible due to latency or computational overhead. The partitioning can be interpreted as randomized load balancing or hashing of resources to dispatchers. The partition is redrawn at random at every time step.
    \item Third, each resource is queried by exactly one dispatcher per time step. This enforces that a resource is matched based only on a \emph{single} arriving context, as in the original problem. In particular, resources cannot compare across multiple simultaneous arrivals, avoiding the “buffet” effect of $m$-scaling where resources effectively choose their most preferred task from a batch.
\end{enumerate}

We define $(m,n)$-scaling formally below. 

\begin{definition}[$(m,n)$-scaled system]
\label{def:fluid_scaled}
Fix $m,n \in \mathbb{N}$. The \emph{$(m,n)$-scaled system} is defined as follows.

\begin{description}
    \item[Resources.] For each type $r \in [R]$, there are $mn$ resources of type $r$.
    \item[Arrivals.] For each timestep $t$, $m$ i.i.d.\ tasks arrive.
    \item[Distribution.] The $mn$ resources of each type are split into $m$ uniform partitions, such that every resource sees a uniformly chosen arrival and each arrival sees $n$ uniformly chosen resources of each type.  
    \item[Assignment.] {\color{red}An arriving task is assigned exactly $nk$ free resources from its dispatcher's partition if at least $nk$ are free; otherwise it is dropped without any assignment.}
    \item[Service.] {\color{red}Each accepted type-$v$ task draws one duration $D_v\in\{1,\ldots,d\}$, independently across tasks conditional on their types and independently of partitioning and assignment. All $nk$ resources assigned to that task share this draw and return together.}
\end{description}

For any policy $\pi$, let $\Reward_{m,n}(\pi)$ denote the long-run average total reward per timestep under $\pi$ in the $(m,n)$-scaled system. {\color{red}When dispatcher $j$ has at least $nk$ free resources, scaled \Lag\ uses the same LP-calibrated allocation rule with budget $k$ and capacities $c_r=C_r/n$, where $C_r$ is its free count of type $r$, then rounds to exactly $nk$ assignments as specified in Appendix~\ref{app:lag-ties}. Writing $a$ for the fractional allocation and $\hat a$ for the rounded counts, this rule guarantees
\[
  0\le\hat a_r\le C_r,\qquad \sum_r\hat a_r=nk,
  \qquad \left|\frac{\hat a_r}{n}-a_r\right|<\frac1n.
\]
At $n=1$, it recovers Definition~\ref{def:lag}.}
\end{definition}

In the remainder of the section we show \Lag\ is asymptotically optimal in the
$(m,n)$-scaled loss system: Its normalized reward converges to
$\mathrm{OPT}(\text{OCC-LP}(0))$ (achievability), and no work-conserving policy exceeds
this value in the limit (upper bound). The upper bound combines
Proposition~\ref{prop:lp_upper_bound}, which bounds any policy by
$\mathrm{OPT}(\text{OCC-LP}(B_{m,n}))$, with the fact that blocking vanishes
(Proposition~\ref{prop:congestion}), so the bound tends to
$\mathrm{OPT}(\text{OCC-LP}(0))$.
\begin{theorem}[Asymptotic optimality of \Lag]
\label{thm:asymptotic_optimality}
Under Assumptions~\ref{assum:underload} and \ref{assum:gap}, in the $(m,n)$-scaled loss
system,
\[
  \frac{\Reward_{m,n}(\Lag)}{mn} \;\longrightarrow\; \mathrm{OPT}(\text{\upshape OCC-LP}(0))
  \qquad \text{as } m,n \to \infty,
\]
and $\displaystyle\limsup_{m,n\to\infty}\frac{\Reward_{m,n}(\pi)}{mn}
\;\le\; \mathrm{OPT}(\text{\upshape OCC-LP}(0))$ for every work-conserving policy $\pi$.
Hence \Lag\ is asymptotically optimal.
\end{theorem}

Below we state the lemmas we use to prove Theorem~\ref{thm:asymptotic_optimality}. Proofs are deferred to Appendix~\ref{app:meanfield}.
\subsection{Lemmas Concerning the Waterfill Function}\label{sec:waterfill}

The proof of Theorem~\ref{thm:asymptotic_optimality} relies on two properties of the OCC-LP optimum: its waterfill structure and the Lipschitz continuity of the waterfill map.

\begin{revblock}
{\color{red}We write $\waterfill(b,c;\iota,w)$ for the fractional allocation that fills capacities $c$ in decreasing index order with budget $b$, using capped proportional allocation with weights $w$ among zero-index servers. Its exact three-stage definition is given in Appendix~\ref{app:lag-ties} (Definition~\ref{def:waterfill}). This is the allocation map used by both the finite and fluid versions of \Lag.}
\begin{lemma}[$4$-Lipschitz]\label{lem:waterfill_lipschitz}
For fixed $b,\iota,w,\pi$, the map $c\mapsto\PW(b,c;\iota,w)$ on $[0,1]^R$ satisfies $\|\PW(b,c')-\PW(b,c)\|_1\le 4\|c'-c\|_1$.
\end{lemma}
\end{revblock}

\begin{revblock}
\begin{lemma}[Fluid fixed point of \Lag]\label{lem:lp-optimal}
Let $(u^*,\phi^*)$, $\ell^*$ and $w$ be as in Definition~\ref{def:lag}.
\begin{enumerate}
  \item[(a)] For every $v\in[V]$ and $r\in[R]$, $u^*_{vr}=\lambda_v\,\waterfill_r\bigl(k,\phi^*;\,L(v,\cdot),\,w_v\bigr)$.
  \item[(b)] Define $\hat x\in\mathcal{X}$ by $\hat x_{r,0}=\phi^*_r$ and $\hat x_{r,i}=\sum_v\Pr(D_v\ge i)\,u^*_{vr}$ for $i=1,\dots,d$. Then $u^{\Lag}(\hat x)=u^*$, $G(\hat x)=\hat x$, and $\sum_{v,r}s_{vr}u^{\Lag}_{vr}(\hat x)=\mathrm{OPT}(\text{OCC-LP}(0))$.
\end{enumerate}
\end{lemma}
\end{revblock}

\subsection{The Fluid Differential Equations}\label{sec:fluid-ode}

Let $\mathcal{X} = \bigl\{x \in [0,1]^{R \times (d+1)} : \sum_{i=0}^{d} x_{r,i} = 1 \;\forall\, r \in [R]\bigr\}$
be the fluid state space, where $x_{r,i}$ denotes the fraction of type-$r$ resources
with exactly $i$ steps of service remaining ($i = 0$ means free).

\begin{revblock}
\begin{definition}\label{def:uLag}
For $x\in\mathcal{X}$ and $\phi=(x_{r,0})_r$, the \emph{fluid assignment rate} of $\Lag$ is
\[
  u_{vr}^{\Lag}(x) \;=\; \lambda_v\,\PW_r\bigl(k,\,\phi;\,L(v,\cdot),\,w_v\bigr),
  \qquad w_v:=(w_{vr})_{r\in[R]}.
\]
{\color{red}When $\sum_r x_{r,0}\ge k$, this is the ideal fractional assignment rate before integer rounding. At shortage states, the waterfill partially allocates the available mass only to define a continuous auxiliary map $G$; the stochastic loss system always drops such tasks. Transition approximation will be required only on the well-stocked set $\Xg$ of Lemma~\ref{lem:mn_transition}.}
\end{definition}
\end{revblock}
We define the fluid map $G$ to describe one step of the system when every random quantity equals its average. Its input $x$ gives, for each type $r$, the fraction of type-$r$ resources with $i$ steps of service left. A resource with $i+1$ steps left moves to $i$ steps left. When \Lag\ assigns a free resource to a type-$v$ task, that resource stays busy for $D_v$ steps. The fluid assignment rate $u_{vr}^{\Lag}(x)$ of Definition~\ref{def:uLag} is the fraction of type-$r$ resources that \Lag\ assigns to type-$v$ tasks in one step. The three lines below apply this to free resources ($i=0$), to resources with $i$ steps left for $1 \le i \le d-1$, and to resources with the maximum $d$ steps left.

\begin{definition}[Fluid map]\label{def:fluid_ODE}
  Define the map $G \colon \mathcal{X} \to \mathcal{X}$ by
  \begin{align*}
    G_{r,0}(x)   &= x_{r,0} - \sum_{v} u_{vr}^{\Lag}(x) + x_{r,1}, \\
    G_{r,i}(x)   &= x_{r,i+1} + \sum_{v} \Pr(D_v = i)\cdot u_{vr}^{\Lag}(x),
                    \qquad i = 1, \ldots, d-1, \\
    G_{r,d}(x)   &= \sum_{v} \Pr(D_v = d)\cdot u_{vr}^{\Lag}(x).
  \end{align*}
\end{definition}

The $(m,n)$-scaled system has $mn$ resources of each type, so the fractions $x_{r,i}$ it can realize are multiples of $1/(mn)$. Let $\mathcal{X}^{m,n} \subseteq \mathcal{X}$ be the set of states $x$ with $mn\,x_{r,i} \in \mathbb{N}$ for all $r$ and $i$. This set is finite.

{\color{red}
\begin{lemma}[Transition convergence on well-stocked states]\label{lem:mn_transition}
Under Assumption~\ref{assum:underload}, let $\gamma=R(1-\rho)-k>0$ and
\[
  \Xg=\left\{x\in\mathcal X:\sum_r x_{r,0}\ge k+\gamma/2\right\}.
\]
Let $G^{m,n}$ be the random one-step transition of the scaled loss system under
\Lag. For every $\delta>0$,
\[
  \lim_{m,n\to\infty}\sup_{x\in\mathcal X^{m,n}\cap\Xg}
  \Pr\!\left(\|G^{m,n}(x)-G(x)\|_\infty>\delta\right)=0.
\]
For every bounded continuous $f:[0,1]^{R\times(d+1)}\to\mathbb R$,
\[
  \lim_{m,n\to\infty}\sup_{x\in\mathcal X^{m,n}\cap\Xg}
  |\E[f(G^{m,n}(x))]-f(G(x))|=0.
\]
\end{lemma}
}

For any fixed $m,n$, let $\mu^{m,n}$ denote a stationary distribution of the
$(m,n)$-scaled chain under $\Lag$; at least one exists since $\mathcal{X}^{m,n}$ is finite.

\begin{revblock}
\begin{lemma}[Reward]\label{lem:mn_reward}
  Let $s_{\max}=\max_{v,r}|s_{vr}|$ and $\gamma=R(1-\rho)-k$. For every $m,n$ and every initial state
  $x(0)\in\mathcal{X}^{m,n}$, there exists a stationary distribution
  $\mu^{m,n}\in\mathbb{P}(\mathcal{X}^{m,n})$ of the $(m,n)$-scaled system under $\Lag$ such that
  \[
    \left|\frac{\Reward_{m,n}(\Lag)}{mn}
    - \sum_{v,r} s_{vr}\int u_{vr}^{\Lag}(x)\,d\mu^{m,n}(x)\right|
    \le s_{\max}\left(\frac{4R}{\sqrt{n}}+\frac{R}{n}+k\,p_n
      +2k\,e^{-\gamma^2 m/(2k^2 d)}\right),
  \]
  where $p_n=R\,e^{-n\gamma^2/(2R^2)}$.
\end{lemma}
\end{revblock}

\subsection{Congestion vanishes}\label{sec:congestion}

The upper-bound half of Theorem~\ref{thm:asymptotic_optimality} needs the blocking
probability to vanish. We show that in the mean-field limit an arriving task almost never
finds fewer than $k$ servers free. Let
$X^{(j)}_{r,0}$ denote the fraction of type-$r$ servers that dispatcher $j$ sees free, and
$x_{r,0}$ the fraction free in the whole system. A task at dispatcher $j$ is short of
servers exactly when $\sum_r X^{(j)}_{r,0} < k$.

Fix any work-conserving policy $\pi$, and let $\mu^{m,n}_\pi$ be a stationary distribution
of the $(m,n)$-scaled chain under $\pi$. We argue in two steps: First that the whole system
keeps plenty of servers free, then that each dispatcher's sample is close enough to the
average to inherit this. The first step writes $\gamma = R(1-\rho) - k$ for the expected number of free servers
in excess of the $k$ that one task needs; Assumption~\ref{assum:underload} makes $\gamma > 0$. The
total number of free servers stays above $k$ with probability tending to one.

\begin{lemma}\label{lem:free-conc}
Let $\pi$ be any work-conserving policy with stationary distribution $\mu^{m,n}_\pi$.
Under Assumption~\ref{assum:underload}, with $\gamma = R(1-\rho)-k > 0$,
\[
  \Pr_{X\sim\mu^{m,n}_\pi}\Bigl[\,\textstyle\sum_r x_{r,0}\le k+\tfrac{\gamma}{2}\Bigr]
  \;\to\;0 \qquad\text{as } m,n\to\infty.
\]
\end{lemma}

The second step accounts for the fact that a dispatcher sees only a sample of $n$
servers per type, not the whole system. The next proposition shows the sample is
close enough to the mean that a task almost never finds itself short.

\begin{proposition}\label{prop:congestion}
Under Assumption~\ref{assum:underload}, for any work-conserving policy the stationary
probability that an arriving task finds fewer than $k$ free servers among those it
observes, $P^{\mathrm{block}}_{m,n}$, tends to $0$ as $m,n\to\infty$.
\end{proposition}
\subsection{Proof of Optimality}\label{sec:mf-proof}

\begin{assumption}[Global Attractor Property]\label{assum:gap}
  There exists a unique $x^*\in\mathcal{X}$ such that for every initial $x(0)\in\mathcal{X}$,
  \[
    G^t(x(0)) \;\to\; x^* \qquad\text{as }t\to\infty.
  \]
\end{assumption}

As noted in~\cite{verloop2016asymptotic}, this global attractor property was imposed
in~\cite{weber1990index} to establish the asymptotic optimality of Whittle's index policy
in a setting with a fixed population of symmetric bandits. In practice, this property is
typically difficult to verify analytically and is usually established through numerical evidence.

\begin{lemma}\label{lem:4_mu_convergence}
  {\color{red}Under Assumptions~\ref{assum:underload} and~\ref{assum:gap}}, for any sequence $((m_k,n_k))_{k=1}^\infty$ with
  $m_k,n_k\to\infty$, the stationary distributions $\mu^{m_k,n_k}$ converge weakly to
  $\delta_{x^*}$.
\end{lemma}

\begin{proof}[Proof sketch of Theorem~\ref{thm:asymptotic_optimality}]
We sketch both halves; full details are in Appendix~\ref{app:meanfield}.

\emph{Achievability.} By Lemma~\ref{lem:mn_reward}, the normalized reward of \Lag\ equals
$\sum_{v,r}s_{vr}\int u^{\Lag}_{vr}(x)\,d\mu^{m,n}(x)$ up to an {\color{red}$o(1)$ error as $m,n\to\infty$}, so it
is governed by the stationary law $\mu^{m,n}$. Assumption~\ref{assum:gap} and
Lemma~\ref{lem:4_mu_convergence} pin this law down: $\mu^{m,n}$ concentrates on the unique
fluid fixed point $x^*$, so the reward converges to $\sum_{v,r}s_{vr}u^{\Lag}_{vr}(x^*)$. It
remains to identify $x^*$. Assembling a fluid state from the waterfill-structured OCC-LP
optimum of Lemma~\ref{lem:lp-optimal} yields a fixed point of the fluid map $G$
(Definition~\ref{def:fluid_ODE}); uniqueness (Assumption~\ref{assum:gap}) forces $x^*$ to be
this state, so $u^{\Lag}(x^*)=u^*$ and the reward converges to
$\mathrm{OPT}(\text{OCC-LP}(0))$.

\emph{Upper bound.} For any work-conserving policy $\pi$,
Proposition~\ref{prop:lp_upper_bound} bounds its normalized reward by
$\mathrm{OPT}(\text{OCC-LP}(B_{m,n}))$, where $B_{m,n}$ is the blocking probability of the
$(m,n)$-scaled system. Blocking is exactly the shortage event of
Proposition~\ref{prop:congestion}, so $B_{m,n}=P^{\mathrm{block}}_{m,n}\to0$. The OCC-LP
value is continuous in the demand $k\lambda_v(1-B)$ at $B=0$. Strict underload
(Assumption~\ref{assum:underload}) puts the demand at $B=0$ in the interior of the set of
demands for which the OCC-LP is feasible, and the value is continuous at every such
interior point. Hence $\mathrm{OPT}(\text{OCC-LP}(B_{m,n}))\to\mathrm{OPT}(\text{OCC-LP}(0))$. The two bounds
meet, so \Lag\ is asymptotically optimal.
\end{proof}

\section{Experiments}
\label{sec:experiments}

We evaluate \Lag\ on six instances spanning real peer-review datasets, synthetic
call-center routing, and structured synthetic regimes, comparing against Greedy, Random,
ALG-SC-LP~\cite{dickerson2021allocation}, and Whittle~\cite{gurushankar2026reviewer}
(Section~\ref{sec:policies}). Our main findings are:
\begin{itemize}
  \item \Lag\ achieves 85.6--94.5\% of OCC-UB across all instances at load $\rho = 0.7$,
    and earns at least as much reward as every other online policy on every instance
    (Table~\ref{tab:performance}). The appendix reports all other loads
    (Table~\ref{tab:perf_full}).
  \item On real peer-review datasets, \Lag\ outperforms Greedy by 6--10 percentage
    points at $\rho = 0.7$; on synthetic instances the gap reaches up to
    23 percentage points (36\% relative).
  \item In the mean-field limit, \Lag\ approaches OCC-UB on all six instances
    we test, in line with
    Theorem~\ref{thm:asymptotic_optimality}
    (Section~\ref{sec:meanfield_convergence}).
\end{itemize}

\subsection{Policies Evaluated}\label{sec:policies}

We compare the following policies and reference bounds throughout Section~\ref{sec:experiments}.

\paragraph{\Lag\ (our policy).}
The Lagrangian index policy defined in Section~\ref{sec:index_policy}.
Upon arrival of a type-$v$ task, \Lag\ assigns the $k$ available servers with the
largest Lagrangian index values
$L(v, r) = s_{vr} - \ell_v^* - \mathbb{E}[D_v]\,q_r^*$ (Definition~\ref{def:lag}).
The dual variables $(\ell^*, q^*)$ are computed once by solving the OCC-LP offline;
no further LP solves are needed at runtime.

\paragraph{Baselines.}
\begin{description}
  \item[Greedy:] At each step, assigns the $k$ available servers with the largest
    current scores $s_{v_t r}$, ignoring future demand entirely.
    Variants of this approach are used in automated online assignment
    systems~\cite{fang2007probabilistic,ZHAO2022103028}.

  \item[Random:] Assigns $k$ servers chosen uniformly at random from all currently
    available servers, ignoring scores entirely.
    Serves as a lower-bound sanity check.

  \item[ALG-SC-LP~\cite{dickerson2021allocation}:] The safe-choice LP policy
    of~\cite{dickerson2021allocation}.
    It solves the KIID LP offline for a fractional allocation $x^*_{vr}$, and at each
    step selects server~$r$ with probability proportional to $x^*_{vr}$ among available
    servers, defaulting to Greedy when no available server has positive LP weight.

  \item[Whittle~\cite{gurushankar2026reviewer}:] A closed-form index policy derived via the
    subsidy-of-indifference approach for restless multi-armed bandits~\cite{whittle1988restless}.
    The index is
    \[
      W(v,r) = s_{vr} - \mathbb{E}[D_v]\sum_{v'} \lambda_{v'} \max(s_{v'r}-s_{vr},\,0).
    \]
    This index has a simple closed-form but provides no asymptotic optimality guarantee.
\end{description}

\paragraph{Upper Bounds.}
The following LP optimal values are not online policies; they are shown as reference
lines in all figures.
\begin{description}
  \item[OCC-UB:] Optimal value of the OCC-LP (Section~\ref{sec:upper-bound-LP}),
    an upper bound on the expected reward of any work-conserving policy
    (Proposition~\ref{prop:lp_upper_bound}).
    At finite scale we plot $\mathrm{OPT}(\text{OCC-LP},\,B{=}\hat{B})$, evaluated
    at the empirical blocking rate $\hat{B}$ measured during simulation.

  \item[KIID-UB:] {\color{red}Optimal value of the KIID LP of~\cite{dickerson2021allocation}.}
\end{description}

\subsection{How Does Performance vary with Load}\label{sec:performance_with_load}

We measure reward as a function of system load
$\rho = \lambda k \mathbb{E}[D] / R$,
where $\lambda = \sum_v \lambda_v$ is the total arrival rate.
We sweep $\rho$ by varying $\lambda$ while holding all other parameters fixed.
At each load point we solve the OCC-LP once and compute \Lag\ indices from the resulting
dual variables; all other policies use only the score matrix and, where applicable,
their respective LPs.
We test six instances across three domains: two from real peer-review data, two synthetic call-center instances, and two analytical instances.
Table~\ref{tab:performance} summarizes results at $\rho = 0.7$; the full load sweep appears in Table~\ref{tab:perf_full} in the appendix.

\begin{figure}[!ht]
  \centering
  \begin{subfigure}[b]{0.48\linewidth}
    \includegraphics[width=\linewidth]{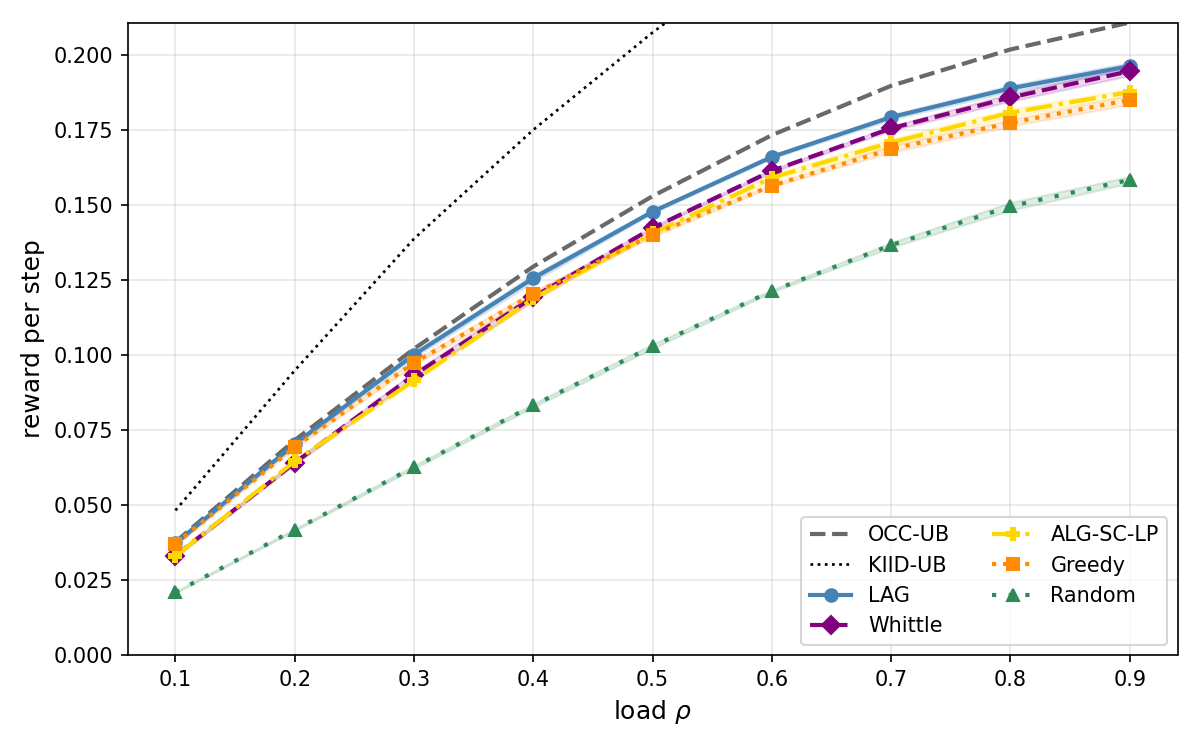}
    \caption{ICLR ($R{=}30$, $V{=}10$, $k{=}3$)}
    \label{fig:load_iclr}
  \end{subfigure}\hfill
  \begin{subfigure}[b]{0.48\linewidth}
    \includegraphics[width=\linewidth]{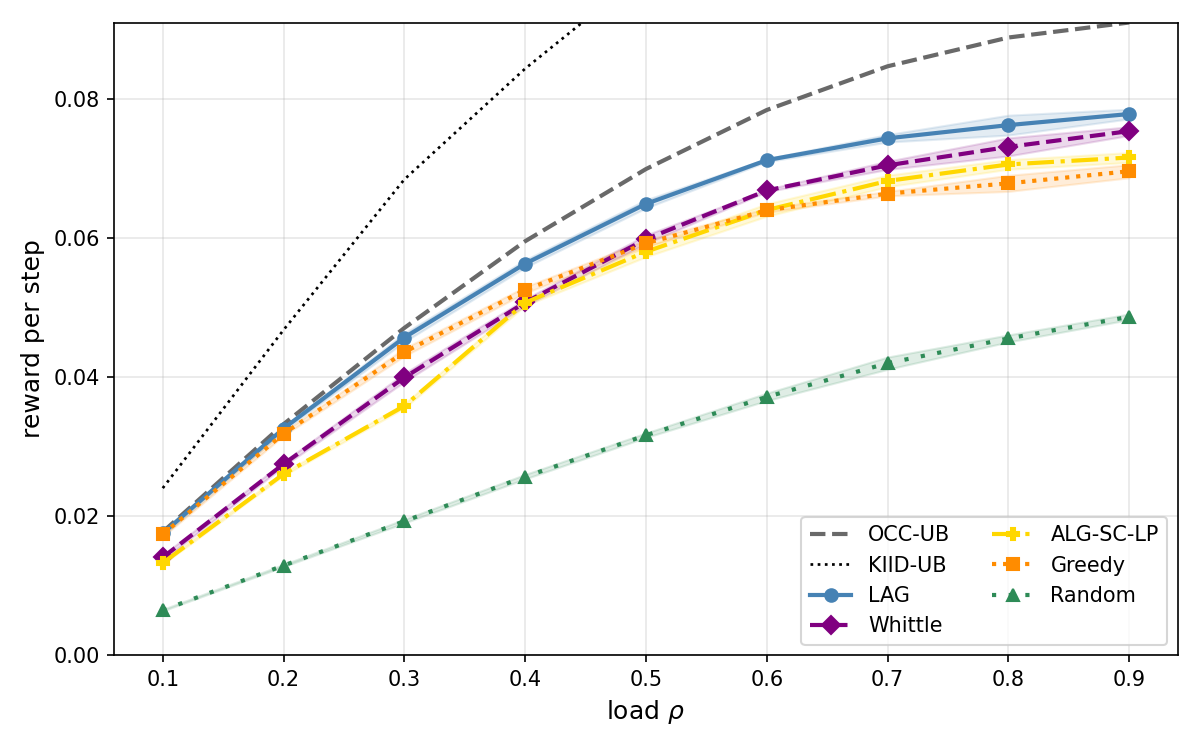}
    \caption{TMLR ($R{=}30$, $V{=}10$, $k{=}3$)}
    \label{fig:load_tmlr}
  \end{subfigure}
  \caption{Reward per step vs.\ load on real peer-review datasets.
    \Lag\ tracks the OCC-UB across all loads.}
  \label{fig:load_peer_review}
\end{figure}

In Figure~\ref{fig:load_peer_review}, arriving tasks represent papers and servers represent
reviewers ($k = 3$).
Match scores are tf-idf cosine similarities~\cite{manning2008introduction} normalized to $[0,1]$;
service times are drawn from the empirical TMLR review-duration distribution~\cite{gurushankar2026reviewer}.
We subsample $R = 30$ reviewers and $V = 10$ papers uniformly at random from each full matrix.
For Figure~\ref{fig:load_iclr} we use the ICLR 2018 similarity score matrix~\cite{xu2019strategyproof};
for Figure~\ref{fig:load_tmlr} we use the TMLR journal similarity matrix~\cite{gurushankar2026reviewer}.
On TMLR, \Lag\ reaches $87.8\%$ of OCC-UB at $\rho = 0.7$ while Greedy falls to $78.4\%$,
a gap that widens at high load.
On ICLR the gap is smaller ($94.5\%$ vs.\ $88.9\%$) but persistent across all loads.

\begin{figure}[!ht]
  \centering
  \begin{subfigure}[b]{0.48\linewidth}
    \includegraphics[width=\linewidth]{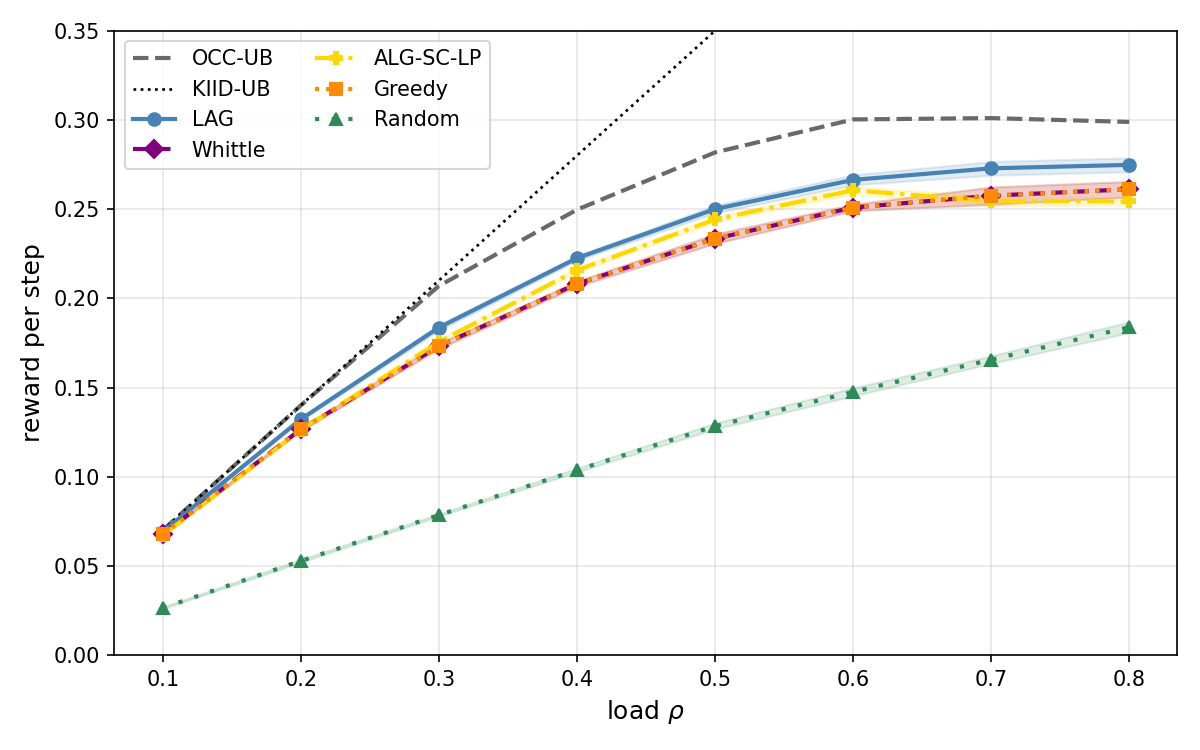}
    \caption{Specialist ($R{=}7$, $V{=}6$, $k{=}1$)}
    \label{fig:load_specialist}
  \end{subfigure}\hfill
  \begin{subfigure}[b]{0.48\linewidth}
    \includegraphics[width=\linewidth]{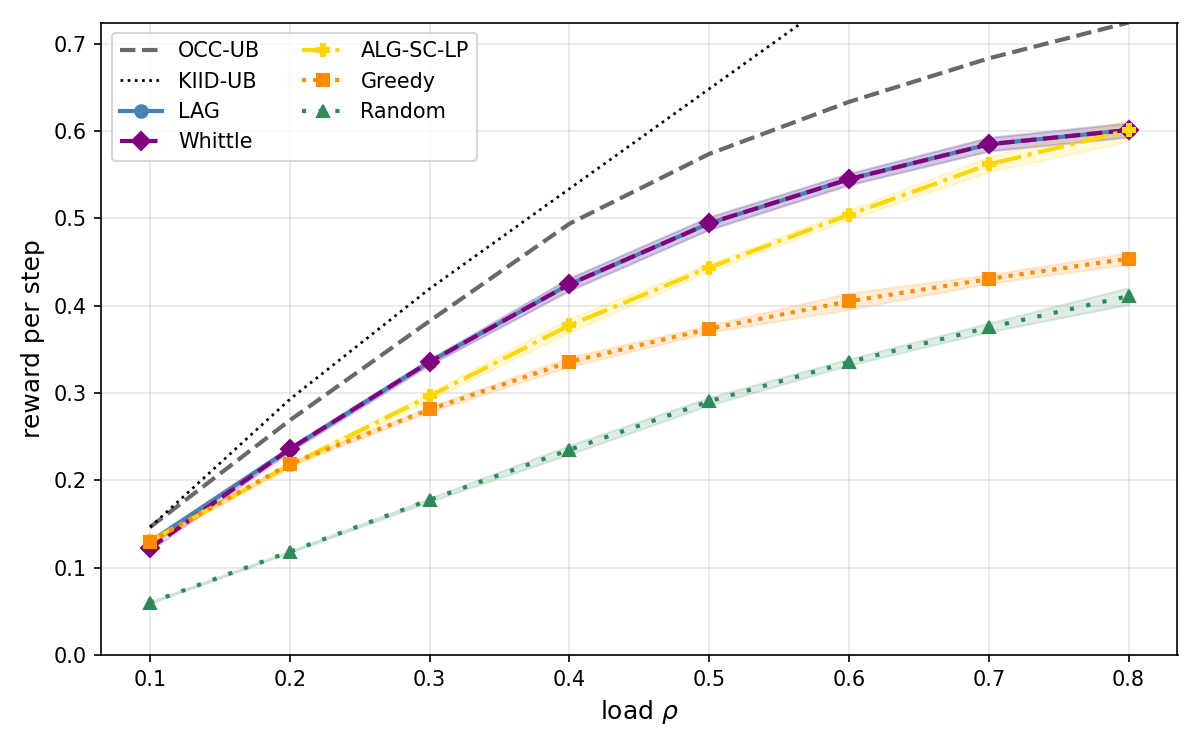}
    \caption{Priority ($R{=}8$, $V{=}2$, $k{=}1$)}
    \label{fig:load_priority}
  \end{subfigure}
  \caption{Reward per step vs.\ load on synthetic call-center instances ($k{=}1$, geometric service).}
  \label{fig:load_callcenter}
\end{figure}

Figure~\ref{fig:load_callcenter} presents two synthetic call-center instances with $k = 1$
and geometric service times.

The \textbf{Specialist} instance ($R = 7$, $V = 6$, $\mathbb{E}[D] = 10$) has two generalists
that score $1.0$ on all six call types and five specialists that score $1.0$ only on their
native routine type.
A hard type with no specialist arrives three times as often as each routine type.
For a routine call, both generalists and the native specialist score $1.0$; Greedy breaks
the three-way tie at random, wasting generalist capacity that only generalists can provide
to hard calls.
At $\rho = 0.7$, Greedy and Whittle both reach $85.5\%$ of OCC-UB; \Lag\ attains $90.6\%$
by routing routine calls to native specialists via the per-type dual prices $\ell_v^*$.

The \textbf{Priority} instance ($R = 8$, $V = 2$, $\mathbb{E}[D] = 20$) has two priority
handlers scoring $[1.8,\,8.0]^\top$ on (standard, priority) calls and six standard handlers
scoring $[1.0,\,0.2]^\top$.
Standard calls arrive $70\%$ of the time.
Because priority handlers outscore standard handlers on standard calls, Greedy assigns them
first to standard arrivals, blocking capacity for high-reward priority calls.
At $\rho = 0.7$, \Lag\ and Whittle both achieve $85.6\%$ of OCC-UB while Greedy falls
to $63.0\%$.

\begin{figure}[!ht]
  \centering
  \begin{subfigure}[b]{0.48\linewidth}
    \includegraphics[width=\linewidth]{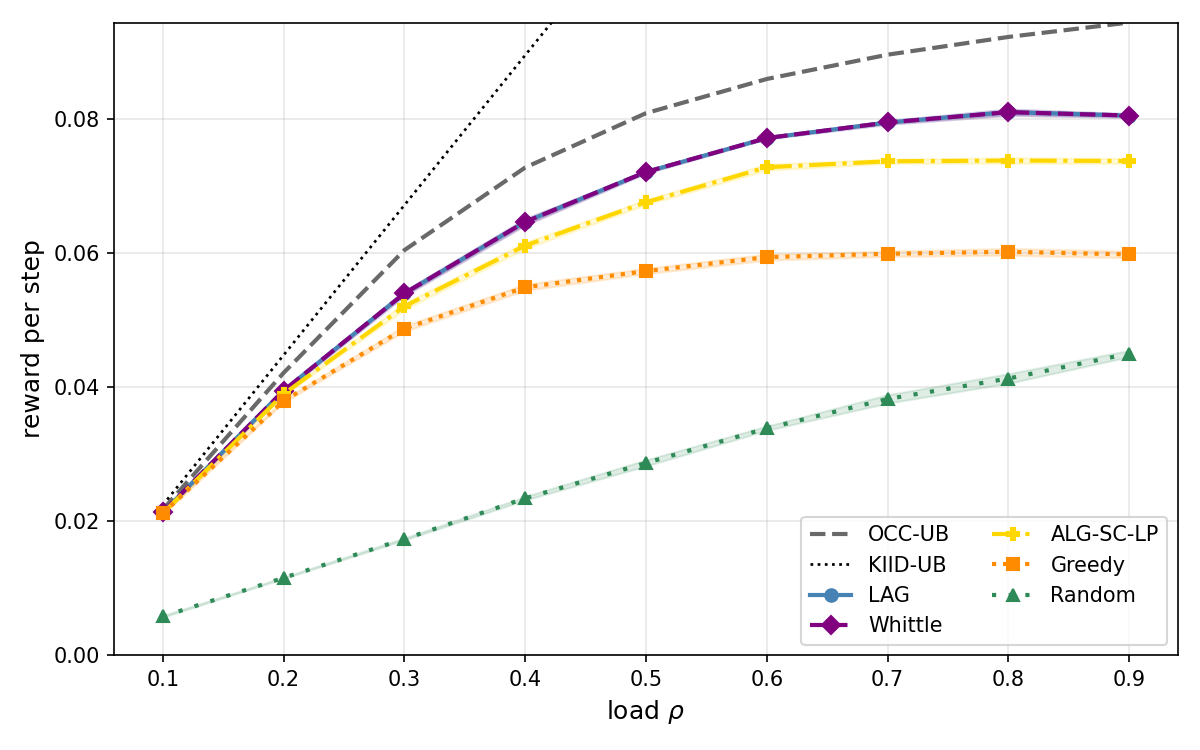}
    \caption{Unfriendly ($R{=}10$, $V{=}32$, $k{=}1$)}
    \label{fig:load_unfriendly}
  \end{subfigure}\hfill
  \begin{subfigure}[b]{0.48\linewidth}
    \includegraphics[width=\linewidth]{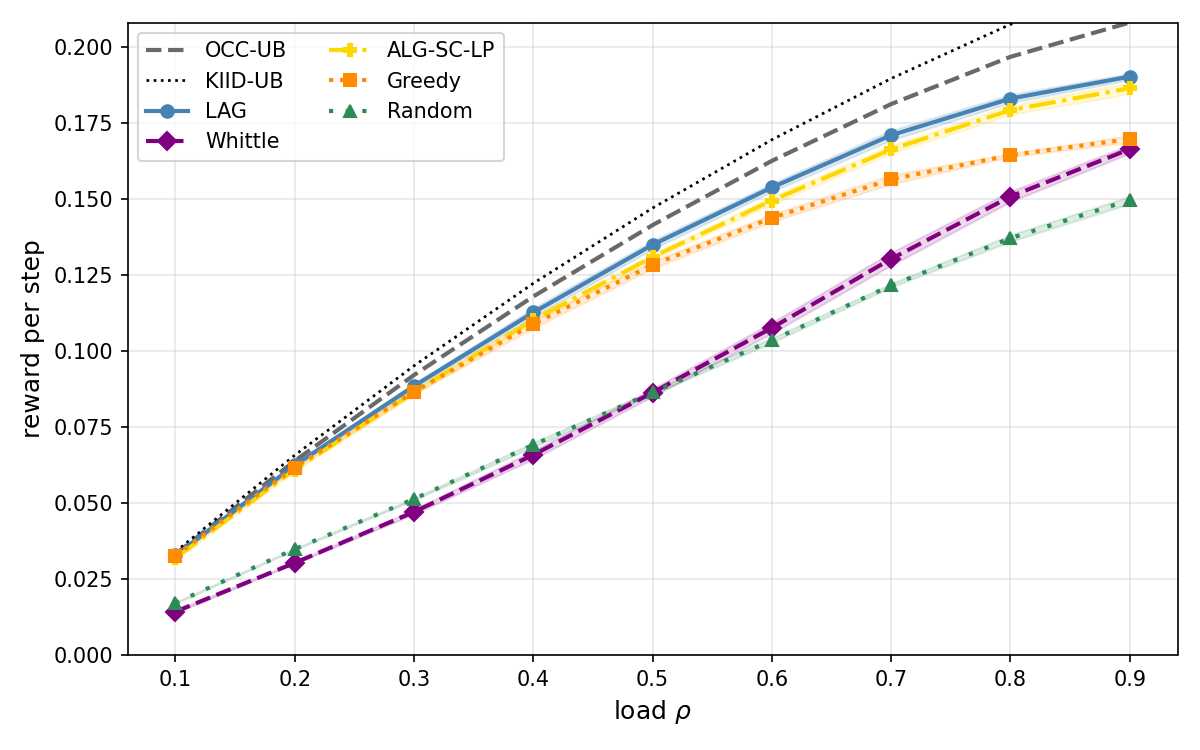}
    \caption{Low-rank ($R{=}30$, $V{=}10$, $k{=}1$)}
    \label{fig:load_lowrank}
  \end{subfigure}
  \caption{Reward per step vs.\ load on synthetic instances.
    On Unfriendly, Greedy wastes high-quality servers on poor-match tasks.
    On Low-rank, Whittle falls more than 22 percentage points below \Lag.}
  \label{fig:load_synthetic}
\end{figure}

Figure~\ref{fig:load_synthetic} presents two synthetic instances designed to expose
complementary failure modes.
Both use deterministic service ($D = \mathbb{E}[D] = 43$ time steps) and $k = 1$.

The \textbf{Unfriendly} instance ($R = 10$, $V = 32$) has $g = 5$ \emph{good} server
types and $b = 5$ \emph{dummy} server types.
Good server~$r \in [g]$ scores $s_{vr} = 1.0$ for the subset of paper types it is
well-matched to, and
$s_{vr} = 0.01$ otherwise; dummy servers score $0$ on all paper types.
The $V = 2^g = 32$ task types cover all binary patterns of the $g$ good servers,
arriving uniformly.
Because the near-zero score of $0.01$ is strictly positive, Greedy assigns a good server
even when the match is poor, burning capacity rather than absorbing the arrival with a
dummy server and holding the good server for a future high-reward match.
At $\rho = 0.7$, Greedy achieves only $66.8\%$ of OCC-UB while \Lag\ attains $88.7\%$,
a relative improvement of $33\%$.
\Lag\ and Whittle coincide on this instance ($k = 1$): Both correctly identify that the
value of withholding a well-matched server outweighs the immediate reward of a poor match.

The \textbf{Low-rank} instance ($R = 30$, $V = 10$) uses a rank-1 score matrix
$s_{vr} = a_v b_r$, where $a_v = (v{-}1)/(V{-}1)$ and $b_r = (r{-}1)/(R{-}1)$ are linearly
spaced in $[0, 1]$ over $v \in [V]$, $r \in [R]$, and task types arrive uniformly.
\Lag\ performs best ($94.3\%$ of OCC-UB at $\rho = 0.7$),
followed by ALG-SC-LP ($91.8\%$) and Greedy ($86.3\%$).
Whittle falls to $71.9\%$, more than 22 percentage points below \Lag\ and more than
14 points below Greedy.

\begin{table}[!ht]
  \centering
  \caption{Reward as \% of OCC-UB at load $\rho = 0.7$.}
  \label{tab:performance}
  \begin{tabular}{lrrrrr}
    \toprule
    Instance & \Lag & Whittle & ALG-SC-LP & Greedy & Random \\
    \midrule
    TMLR        & 87.8 & 83.2 & 80.5 & 78.4 & 49.6 \\
    ICLR        & 94.5 & 92.5 & 90.0 & 88.9 & 72.0 \\
    \midrule
    Specialist  & 90.6 & 85.5 & 84.6 & 85.5 & 55.0 \\
    Priority    & 85.6 & 85.6 & 82.2 & 63.0 & 54.9 \\
    \midrule
    Unfriendly  & 88.7 & 88.7 & 82.2 & 66.8 & 42.7 \\
    Low-rank    & 94.3 & 71.9 & 91.8 & 86.3 & 67.1 \\
    \bottomrule
  \end{tabular}
\end{table}

\subsection{Convergence to the Mean-Field Limit}\label{sec:meanfield_convergence}

We test Theorem~\ref{thm:asymptotic_optimality} empirically.
Using the six instances at load $\rho = 0.8$, we scale the system under the
$(m,n)$-scaling with $m = \lfloor\sqrt{n}\rfloor$, sweeping
$n \in \{1, 4, 16, 64, 256, 1024, 4096, 16384, 65{,}536\}$.
Simulation length scales adaptively as $T \propto 1/m$ so that total computational
work stays roughly constant across $n$ values.
Each data point is the mean of 3 independent trials; shaded bands show
${\pm}1.96\,\hat{\sigma}/\sqrt{3}$ confidence intervals.

\begin{figure}[!ht]
  \centering
  \includegraphics[width=0.55\linewidth]{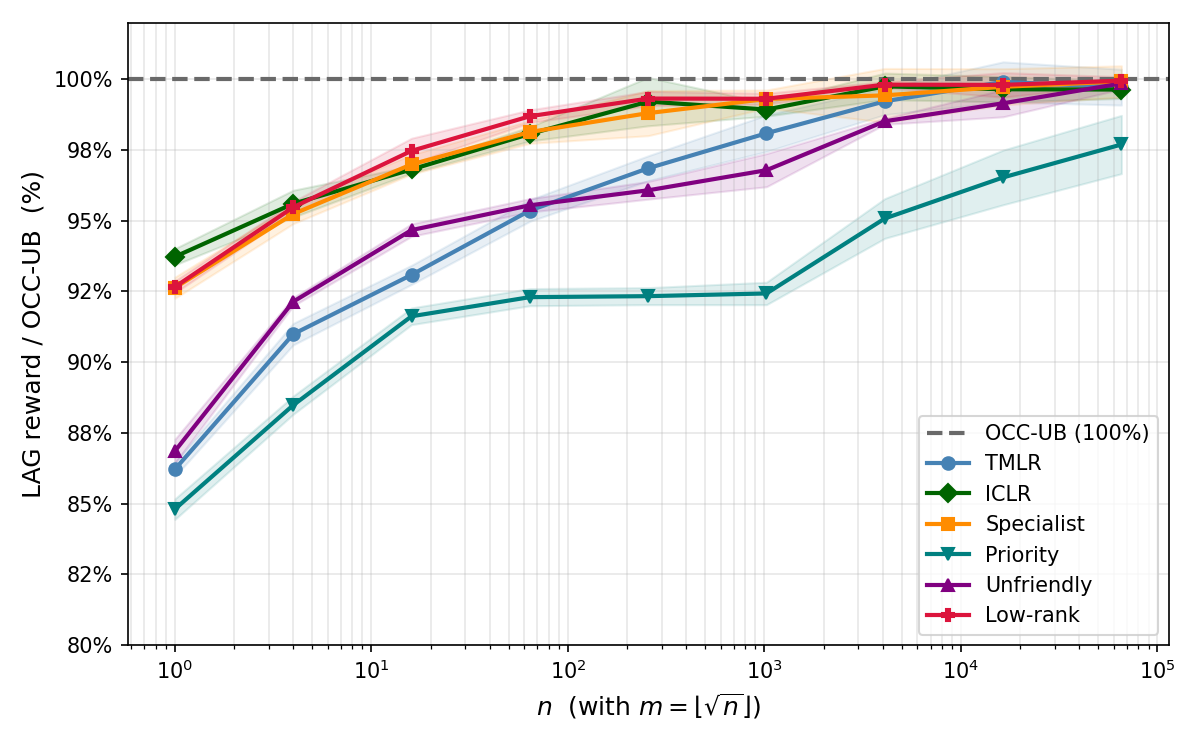}
  \caption{\Lag's reward per server copy, $\text{reward}/(mn)$, as a percentage of
    OCC-UB vs.\ $n$ under the $(m,n)$-scaling ($m = \lfloor\sqrt{n}\rfloor$).
    OCC-UB is $\mathrm{OPT}(\text{OCC-LP}(\hat B))$ at the measured blocking rate $\hat B$.
    All six instances approach OCC-UB as $n$ grows, consistent with
    Theorem~\ref{thm:asymptotic_optimality}.}
  \label{fig:meanfield_convergence}
\end{figure}

Figure~\ref{fig:meanfield_convergence} plots \Lag's reward per server copy as a
percentage of OCC-UB for all six instances on a log~$n$ axis.
At each $n$, OCC-UB is $\mathrm{OPT}(\text{OCC-LP}(\hat B))$, where $\hat B$ is the blocking
rate we measure in the simulation, as in Section~\ref{sec:policies}.
The measured blocking rate falls from between $0.104$ and $0.133$ at $n = 1$ on all
instances except Low-rank ($0.010$) to at most $0.001$ at $n = 65{,}536$.
At $n = 65{,}536$ ($m = 256$), \Lag\ reaches $99.7\%$ on TMLR, $99.6\%$ on ICLR,
$99.9\%$ on Specialist, $97.7\%$ on Priority, $99.8\%$ on Unfriendly, and $99.9\%$ on Low-rank.
Convergence is fastest on Low-rank, which crosses $99\%$ at $n = 256$.
It is slowest on Priority, which stays between $91.6\%$ and $96.5\%$ for $n$ from $16$ to
$16{,}384$ before reaching $97.7\%$ at $n = 65{,}536$.

\FloatBarrier
\section{Conclusion}
\label{sec:conclusion}

We study online assignment of reusable, quality-differentiated resources, where each
arriving task must be matched immediately to $k$ servers that then become busy for a random
duration. We cast the problem as a contextual restless multi-armed bandit and derive the
occupancy-measure LP, whose optimum upper-bounds the reward of any work-conserving policy.
The dual of this LP decouples the routing problem across servers and yields \Lag, a
closed-form index policy: Each task-type and server pair is scored by its raw match value
less an opportunity-cost term that reflects the value of keeping that server free for future
arrivals. To analyze \Lag\ we introduce an $(m,n)$-scaling regime, in which the number of
copies of each server and the number of arrivals per step grow together; scaling the
arrivals concentrates the context distribution and restores the deterministic fluid limit
that a contextual bandit otherwise lacks. In this regime we prove \Lag\ is asymptotically
optimal, attaining the LP bound. Empirically, on real peer-review traces and synthetic
instances, \Lag\ achieves 83--99\% of the LP upper bound and improves over the natural
greedy policy by up to 12.3\% on peer-review data, including well outside the regime where our
optimality theory applies.

\bibliographystyle{ACM-Reference-Format}
\bibliography{bibliography}

\appendix
\section{Tie-Breaking and Rounding Rules for \Lag}\label{app:lag-ties}
\begingroup\color{red}
The rules below complete Definition~\ref{def:lag} and its $(m,n)$-scaled version in Definition~\ref{def:fluid_scaled}. Fix the primal and dual LP solutions once, and fix a total order $\pi$ on $[R]$ for all equal-index and equal-remainder ties. The tie weights $w_{vr}=u^*_{vr}/\lambda_v$ are the LP assignment rates divided by the type-$v$ arrival rate, so $\sum_r w_{vr}=k$. This choice is used in the fluid fixed-point proof (Lemma~\ref{lem:lp-optimal}).

\subsection{Fractional Allocation}
The waterfill ranks servers by index. Among servers with index zero, it splits the remaining budget in proportion to given weights.
\begin{definition}[Waterfill]\label{def:waterfill}
Let $b\ge 0$ be a budget, $c\in[0,1]^R$ capacities, $\iota\in\mathbb{R}^R$ an index vector, and $w\in[0,\infty)^R$ a weight vector. Let $\pi$ be a fixed total order on $[R]$ that breaks ties between equal indices. Split $[R]$ into
\[
  P=\{r:\iota_r>0\},\qquad T=\{r:\iota_r=0,\ w_r>0\},\qquad N=[R]\setminus(P\cup T).
\]
The \emph{waterfill} $\waterfill(b,c;\iota,w)\in\mathbb{R}^R$ is computed in three stages.
\begin{enumerate}
  \item \emph{Group $P$.} Fill the capacities of $P$ in decreasing order of $\iota_r$ (ties by $\pi$) until the budget is used or all of $P$ is full. Let $b_1$ be the unused budget.
  \item \emph{Group $T$.} Set $u_r=\min(c_r,\theta w_r)$ for $r\in T$, where $\theta\ge0$ is the smallest number with $\sum_{r\in T}u_r=\min(b_1,\sum_{r\in T}c_r)$. Let $b_2=b_1-\sum_{r\in T}u_r$.
  \item \emph{Group $N$.} Fill the capacities of $N$ in decreasing order of $\iota_r$ (ties by $\pi$) with budget $b_2$.
\end{enumerate}
\end{definition}

All coordinates are initially zero. If $T$ is empty, skip Stage~2 and set $b_2=b_1$. For a type-$v$ task, use $b=k$, $\iota=L(v,\cdot)$, and $w=w_v$. In the original system, set $c_r=\mathbf1[r\in F]$, where $F$ is the free-server set; in a dispatcher's scaled partition, set $c_r=C_r/n$. Drop the task without any assignment if the free count is below $k$ (respectively, $nk$). Otherwise the fractional allocation
\[
  a=\waterfill\bigl(k,c;L(v,\cdot),w_v\bigr)
\]
has $0\le a_r\le c_r$ and $\sum_r a_r=k$.

\subsection{Integer Assignments}
Set $n=1$ for the original system. For either version, first set $\hat a_r=\lfloor na_r\rfloor$, then add one to each of the $nk-\sum_r\lfloor na_r\rfloor$ coordinates with the largest positive fractional remainders of $na_r$, breaking equal-remainder ties by $\pi$. Assign $\hat a_r$ free copies of type $r$; at $n=1$, these are precisely the servers whose rounded allocation is one. The type counts are deterministic given the free capacities and the chosen LP solutions; any free copies within a type may be selected.

Since $\sum_r na_r=nk$ is an integer, there are enough positive fractional remainders to distribute the leftover units. A coordinate with a positive remainder satisfies $\lceil na_r\rceil\le C_r$, since $C_r$ is an integer. Consequently,
\[
  0\le\hat a_r\le C_r,\qquad \sum_r\hat a_r=nk,
  \qquad \left|\frac{\hat a_r}{n}-a_r\right|<\frac1n.
\]
These are the capacity, budget, and rounding-error guarantees used in the mean-field proofs.
\endgroup

\section{Deferred Proofs for the Derivation of \Lag}\label{app:lag-derivation}

\begin{proof}[Proof of Lemma~\ref{lem:bandit-lp-equivalence}]
{\color{red}Throughout this proof, the penalty vector is $\ell=\ell^*$.}
For any Arm-$r$ bandit policy~$\pi$, let $u_{vr}^\pi$ be the long-run rate of
type-$v$ assignments to server~$r$, and let $\phi_r^\pi$ be the long-run free fraction
of server~$r$.
By the argument of Proposition~\ref{prop:lp_upper_bound} applied to a single
server (Little's Law for the capacity constraint, and the martingale bound for the
availability constraint), $(u^\pi, \phi^\pi)$ is feasible for the Arm-$r$ LP.
Thus the average reward of~$\pi$ is
$g_\pi = \sum_v (s_{vr}-\ell_v)\, u_{vr}^\pi \leq q_r^*$.

To show the bound is tight, given the Arm-$r$ LP optimal solution $(u^*, \phi^*)$, construct bandit policy~$\hat\pi$
that, when free and a type-$v$ task arrives, accepts with probability
$u_{vr}^* / (\lambda_v \phi_r^*)$.
The long-run assignment rate under~$\hat\pi$ satisfies:
\[
  u_{vr}^{\hat\pi}
  \;=\; \frac{u_{vr}^*}{\lambda_v \phi_r^*} \cdot \lambda_v \phi_r^{\hat\pi}
  \;=\; \frac{u_{vr}^*}{\phi_r^*} \cdot \phi_r^{\hat\pi}.
\]
Applying Little's Law:
$\phi_r^{\hat\pi}
  = 1 - \sum_v \mathbb{E}[D_v]\, u_{vr}^{\hat\pi}
  = 1 - \tfrac{\phi_r^{\hat\pi}}{\phi_r^*} \sum_v \mathbb{E}[D_v]\, u_{vr}^*
  = 1 - \tfrac{\phi_r^{\hat\pi}}{\phi_r^*}(1 - \phi_r^*)$.
Solving gives $\phi_r^{\hat\pi} = \phi_r^*$, hence $u_{vr}^{\hat\pi} = u_{vr}^*$,
and the average reward of~$\hat\pi$ equals
$\sum_v (s_{vr} - \ell_v)\, u_{vr}^* = q_r^*$.
\end{proof}

\begin{proof}[Proof of Lemma~\ref{lem:bandit-optimal}]
{\color{red}Throughout this proof, the penalty vector is $\ell=\ell^*$.}
We use the average-reward Bellman optimality equation. Let $g$ be the optimal
average reward per step and $h(\cdot)$ a relative value function, normalized so
that $h(0) = 0$ at the free state. A busy server with $x \ge 1$ steps remaining
takes no action and earns no reward, advancing deterministically to state $x-1$,
so $g + h(x) = h(x-1)$; with $h(0)=0$ this gives $h(x) = -g\,x$.

When the server is free and a type-$v$ task arrives, accepting earns
$s_{vr} - \ell_v$ and moves the server to state $D_v$, while rejecting earns
$0$ and keeps it free. Accepting is optimal iff its value is at least that of
rejecting:
\[
  (s_{vr} - \ell_v) + \mathbb{E}[h(D_v)] \;\ge\; h(0),
  \qquad\text{i.e.}\qquad
  s_{vr} - \ell_v \;\ge\; g\,\mathbb{E}[D_v],
\]
since $\mathbb{E}[h(D_v)] = -g\,\mathbb{E}[D_v]$. By
Lemma~\ref{lem:bandit-lp-equivalence}, $g = q_r^*$, giving the stated threshold.
\end{proof}

\begin{revblock}
\begin{lemma}[Prices and complementary slackness]\label{lem:cs}
Suppose Assumption~\ref{assum:underload} holds. Let $\ell^*$ be an optimal dual solution, let $q_r^*=q_r(\ell^*)$, and let $L(v,r)=s_{vr}-\ell_v^*-\E[D_v]q_r^*$. Then:
\begin{enumerate}
  \item[(i)] $q_r^*=\sum_v\lambda_v\max\{L(v,r),0\}$ for every $r\in[R]$.
  \item[(ii)] For every $(u,\phi)$ feasible for $\text{OCC-LP}(0)$,
  \[
    \sum_r q_r^*+k\sum_v\lambda_v\ell_v^*-\sum_{v,r}s_{vr}u_{vr}
    =\sum_{L(v,r)>0}L(v,r)\bigl(\lambda_v\phi_r-u_{vr}\bigr)+\sum_{L(v,r)<0}\bigl(-L(v,r)\bigr)u_{vr}.
  \]
  \item[(iii)] A feasible $(u,\phi)$ is optimal if and only if $L(v,r)>0\Rightarrow u_{vr}=\lambda_v\phi_r$ and $L(v,r)<0\Rightarrow u_{vr}=0$.
\end{enumerate}
\end{lemma}
\begin{proof}
(i) By Lemma~\ref{lem:bandit-lp-equivalence}, $q_r^*$ is the optimal average reward $g$ of the Arm-$r$ bandit with penalties $\ell^*$. By the proof of Lemma~\ref{lem:bandit-optimal}, the relative value function is $h(x)=-gx$ with $h(0)=0$. The Bellman equation at the free state reads $g=\sum_v\lambda_v\max\{s_{vr}-\ell_v^*+\E h(D_v),0\}$, and $\E h(D_v)=-g\E[D_v]$. This is $g=\sum_v\lambda_v\max\{L(v,r),0\}$.

(ii) Substitute $s_{vr}=L(v,r)+\ell_v^*+\E[D_v]q_r^*$. Constraint~\eqref{lp:demand} gives $\sum_ru_{vr}=k\lambda_v$ and constraint~\eqref{lp:capacity} gives $\sum_v\E[D_v]u_{vr}=1-\phi_r$. So $\sum_{v,r}s_{vr}u_{vr}=\sum_{v,r}L(v,r)u_{vr}+k\sum_v\lambda_v\ell_v^*+\sum_rq_r^*(1-\phi_r)$. The left side of (ii) therefore equals $\sum_rq_r^*\phi_r-\sum_{v,r}L(v,r)u_{vr}$. By (i), $q_r^*\phi_r=\sum_v\lambda_v\phi_r\max\{L(v,r),0\}$. Split the pairs by the sign of $L(v,r)$.

(iii) By Assumption~\ref{assum:underload}, $\text{OCC-LP}(0)$ is feasible (the point $u_{vr}=k\lambda_v/R$, $\phi_r=1-\rho$ satisfies all constraints), so strong duality holds. Hence $(u,\phi)$ is optimal if and only if the left side of (ii) is $0$. By $0\le u_{vr}\le\lambda_v\phi_r$, every term on the right of (ii) is nonnegative, so the sum is $0$ if and only if each term is $0$.
\end{proof}
\end{revblock}

\section{Deferred Proofs for Mean Field Optimality}
\label{app:meanfield}

\begin{revblock}
\begin{proof}[Proof of Lemma~\ref{lem:waterfill_lipschitz}]
We use two facts for each of the two kinds of fill, priority fill (Stages 1 and 3) and weighted fill (Stage 2). Both output $u$ with $\sum_r u_r=\min(b,\sum_r c_r)$.
\begin{itemize}
  \item[(F1)] \emph{Budget.} If $b\le b'$ then $u(b)\le u(b')$ coordinatewise. Hence $\|u(b')-u(b)\|_1=\min(b',\sum c)-\min(b,\sum c)\le b'-b$.
  \item[(F2)] \emph{Capacity.} $\|u(c')-u(c)\|_1\le 2\|c'-c\|_1$ for fixed $b$.
\end{itemize}
For priority fill, (F1) is clear. For (F2), by the triangle inequality along a coordinatewise path it suffices to take $c'=c+\delta e_j$ with $\delta>0$. Allocations before $j$ in the order do not change. The allocation of $j$ increases by at most $\delta$. Allocations after $j$ weakly decrease, and since the total $\min(b,\sum c)$ weakly increases, their total decrease is at most the increase at $j$. So $\|u'-u\|_1\le2\delta$. For weighted fill, (F1) holds because $\theta$ is nondecreasing in $b$ and $u_r=\min(c_r,\theta w_r)$. For (F2), by the triangle inequality along a coordinatewise path inside $[0,1]^T$ it suffices to treat $c'=c+\delta e_j$ with $\delta>0$. Write $\Sigma=\sum_r c_r$ and $g(\theta,c)=\sum_r\min(c_r,\theta w_r)$.

\emph{Case $b\le\Sigma$.} Since $g(\theta,c')\ge g(\theta,c)$, the smallest $\theta'$ with $g(\theta',c')=b$ satisfies $\theta'\le\theta$. So $u'_r\le u_r$ for $r\ne j$. Both totals equal $b$, so $u'_j\ge u_j$ and $\|u'-u\|_1=2(u'_j-u_j)$. Also $u'_j=\min(c_j+\delta,\theta' w_j)\le\min(c_j,\theta w_j)+\delta=u_j+\delta$.

\emph{Case $\Sigma<b\le\Sigma+\delta$.} {\color{red}Here $u=c$ and $\sum_r u'_r=b$. For $r\ne j$, the capacity bound gives $u'_r\le c_r=u_r$. Therefore
\[
  u'_j=b-\sum_{r\ne j}u'_r\ge b-\sum_{r\ne j}c_r>c_j=u_j,
\]
while $u'_j\le c_j+\delta$. Since the total allocation increases from $\Sigma$ to $b$, the total decrease in the other coordinates is at most $u'_j-u_j$. Hence $\|u'-u\|_1\le2(u'_j-u_j)\le2\delta$.}

\emph{Case $b>\Sigma+\delta$.} Then $u=c$, $u'=c'$ and $\|u'-u\|_1=\delta$.

Now take $c,c'\in[0,1]^R$ and let $a,t,n$ be the $\ell_1$ distances between $c$ and $c'$ on $P,T,N$. The unused budgets are $b_1=(b-\sum_P c)^+$ and $b_2=(b_1-\sum_T c)^+$, so $|\Delta b_1|\le a$ and $|\Delta b_2|\le|\Delta b_1|+t\le a+t$. By (F1), (F2) and the triangle inequality, the three stages change by at most $2a$, $|\Delta b_1|+2t\le a+2t$ and $|\Delta b_2|+2n\le a+t+2n$. The sum is at most $4a+3t+2n\le 4\|c'-c\|_1$.
\end{proof}
\end{revblock}

\begin{revblock}
\begin{proof}[Proof of Lemma~\ref{lem:lp-optimal}]
(a) Fix $v\in[V]$, and let $P,T,N$ be the groups of Definition~\ref{def:waterfill} for $\iota=L(v,\cdot)$ and $w=w_v$. By Lemma~\ref{lem:cs}(iii) and the bound~\eqref{lp:bound}:
\begin{enumerate}
  \item if $L(v,r)>0$, then $u^*_{vr}=\lambda_v\phi^*_r$;
  \item if $L(v,r)<0$, then $u^*_{vr}=0$;
  \item if $L(v,r)=0$, then $0\le w_{vr}\le\phi^*_r$.
\end{enumerate}
The demand constraint~\eqref{lp:demand} says $\sum_ru^*_{vr}=k\lambda_v$. By (1), (2), and since $w_{vr}=0$ for tied servers outside $T$,
\begin{equation}\label{eq:split}
  k=\sum_{r\in P}\phi^*_r+\sum_{r\in T}w_{vr},\qquad\text{so}\qquad \sum_{r\in P}\phi^*_r\le k.
\end{equation}
Stage~1 uses capacities $\phi^*$. By~\eqref{eq:split} the budget $k$ fills all of $P$, so the waterfill equals $\phi^*_r=u^*_{vr}/\lambda_v$ on $P$ by (1), and the unused budget is $b_1=\sum_{r\in T}w_{vr}$. In Stage~2, if $T=\emptyset$ then $b_1=0$. Otherwise $b_1>0$. At $\theta=1$, (3) gives $\min(\phi^*_r,w_{vr})=w_{vr}$ for $r\in T$, so $\sum_Tu_r=b_1\le\sum_T\phi^*_r$. For $\theta<1$, $\sum_T\min(\phi^*_r,\theta w_{vr})\le\theta b_1<b_1$. So $\theta=1$ is the smallest solution, the waterfill equals $w_{vr}=u^*_{vr}/\lambda_v$ on $T$, and $b_2=0$. Stage~3 has budget $0$, so the waterfill is $0$ on $N$. Every $r\in N$ has $L(v,r)<0$, or $L(v,r)=0$ and $w_{vr}=0$, so $u^*_{vr}=0$ by (2) and the definition of $w$.

(b) By constraint~\eqref{lp:capacity}, $\sum_{i\ge0}\hat x_{r,i}=\phi^*_r+\sum_v\E[D_v]u^*_{vr}=1$, using $\sum_{i\ge1}\Pr(D_v\ge i)=\E[D_v]$. So $\hat x\in\mathcal{X}$. Since $\hat x_{r,0}=\phi^*_r$, part~(a) gives $u^{\Lag}(\hat x)=u^*$. Since $\Pr(D_v\ge1)=1$, $\hat x_{r,1}=\sum_vu^*_{vr}$, so $G_{r,0}(\hat x)=\phi^*_r-\sum_vu^*_{vr}+\hat x_{r,1}=\hat x_{r,0}$. For $1\le i\le d$, with $\hat x_{r,d+1}:=0$,
$G_{r,i}(\hat x)=\hat x_{r,i+1}+\sum_v\Pr(D_v=i)u^*_{vr}=\hat x_{r,i}$, because $\Pr(D_v\ge i+1)+\Pr(D_v=i)=\Pr(D_v\ge i)$. The reward is $\sum_{v,r}s_{vr}u^*_{vr}=\mathrm{OPT}(\text{OCC-LP}(0))$.
\end{proof}
\end{revblock}

{\color{red}
\begin{proof}[Proof of Lemma~\ref{lem:mn_transition}]
Fix $x\in\mathcal X^{m,n}\cap\Xg$, and write $\phi=(x_{r,0})_r$.
Let $\phi^{(j)}_r$ be the fraction of type-$r$ resources that dispatcher $j$ sees free.
Uniform partitioning gives hypergeometric marginals, so, for $mn>1$,
\begin{equation}\label{eq:hypergeo-moments}
  \E\phi^{(j)}_r=\phi_r,\qquad
  \Var(\phi^{(j)}_r)=\phi_r(1-\phi_r)\frac{m-1}{mn-1}\le\frac1n.
\end{equation}
For $m=n=1$ the partition is deterministic and its variance is zero.
Consequently $\E\|\phi^{(j)}-\phi\|_1\le R/\sqrt n$.
The partitions need not be independent for this estimate.

Write $A_j=\mathbf1[\sum_r\phi^{(j)}_r\ge k]$ for the admission indicator.
Since $\sum_r\phi_r\ge k+\gamma/2$, a shortage requires
$\phi_r-\phi^{(j)}_r>\gamma/(2R)$ for some $r$.
Hoeffding's inequality for sampling without replacement and a union bound give
\[
  \Pr(A_j=0\mid x)\le p_n:=R\exp\!\left(-\frac{n\gamma^2}{2R^2}\right).
\]
This bounds a single dispatcher's shortage probability uniformly on $\Xg$;
we will average errors over dispatchers, so no union bound over $m$ is needed.

Independently of the partitions, draw the $m$ i.i.d.\ task marks $(V_j,D_j)$:
$\Pr(V_j=v,D_j=i)=\lambda_v\Pr(D_v=i)$, with a no-arrival mark of probability
$1-\sum_v\lambda_v$. The duration $D_j$ is one draw for the entire task and is
independent of its assignment conditional on $V_j$. We may also draw it for a
rejected task, since its allocation is zero. Define
\[
  W_v(c)=\waterfill\bigl(k,c;L(v,\cdot),w_v\bigr),\qquad W_\emptyset(c)=0.
\]
Let $B_{jr}$ be the actual number of type-$r$ assignments by dispatcher $j$, divided
by $n$. Thus $B_j=0$ on shortage or no arrival; on admission it is the rounded
allocation of Definition~\ref{def:fluid_scaled}. In particular,
\begin{equation}\label{eq:ujvr}
  U^{(j)}_{vr}=\mathbf1[V_j=v]B_{jr},\qquad
  U^{(j)}_{vr,i}=\mathbf1[V_j=v,D_j=i]B_{jr}.
\end{equation}
All coordinates assigned to the same task use the same $D_j$.

For each $r$, the $4$-Lipschitz property and the rounding bound imply
\[
  |B_{jr}-W_{V_j,r}(\phi)|
  \le 4\|\phi^{(j)}-\phi\|_1+\frac1n+(1-A_j).
\]
On admission this follows by inserting $W_{V_j}(\phi^{(j)})$; on rejection
$B_j=0$ and $0\le W_{V_j,r}(\phi)\le1$.
It follows uniformly in $x\in\mathcal X^{m,n}\cap\Xg$ that
\[
  \E|B_{jr}-W_{V_j,r}(\phi)|\le
  \eta_n:=\frac{4R}{\sqrt n}+\frac1n+p_n.
\]

To account for arrivals and shared service durations together, set
\[
  Y_{jr0}=W_{V_j,r}(\phi),\qquad
  Y_{jri}=W_{V_j,r}(\phi)\mathbf1[D_j=i]\quad(1\le i\le d),
\]
with $Y_{jri}=0$ on a no-arrival mark. For fixed $r,i$, these variables are
independent over $j$, lie in $[0,1]$, and have means
\[
  \E Y_{jr0}=\sum_v u^{\Lag}_{vr}(x),\qquad
  \E Y_{jri}=\sum_v\Pr(D_v=i)u^{\Lag}_{vr}(x)\quad(i\ge1).
\]
Their empirical means therefore satisfy
\[
  \E\left|\frac1m\sum_jY_{jri}-\E Y_{1ri}\right|
  \le\sqrt{\Var\!\left(\frac1m\sum_jY_{jri}\right)}
  \le\frac1{2\sqrt m}.
\]
Independence between resources within a task is neither assumed nor needed:
the averaging here is over independent task marks. This also keeps the
service-type probabilities inside the calculation, without dropping weights
from a signed sum of allocation errors.

The actual transition is
\begin{align*}
  G^{m,n}_{r,0}(x)&=x_{r,0}+x_{r,1}-\frac1m\sum_jB_{jr},\\
  G^{m,n}_{r,i}(x)&=x_{r,i+1}+\frac1m\sum_jB_{jr}\mathbf1[D_j=i],
       \qquad 1\le i\le d,
\end{align*}
where $x_{r,d+1}=0$ and no-arrival contributions are zero.
Comparing these expressions with $G$, insert the corresponding empirical mean
of $Y_{jri}$. Multiplication by $\mathbf1[D_j=i]$ cannot increase the absolute
allocation discrepancy, so for every coordinate,
\[
  \E|G^{m,n}_{r,i}(x)-G_{r,i}(x)|\le\eta_n+\frac1{2\sqrt m}.
\]
Summing over coordinates and applying Markov's inequality yields, for $\delta>0$,
\begin{equation}\label{eq:chernoff-bound}
  \sup_{x\in\mathcal X^{m,n}\cap\Xg}
  \Pr\!\left(\|G^{m,n}(x)-G(x)\|_\infty>\delta\right)
  \le\frac{R(d+1)}{\delta}
       \left(\frac{4R}{\sqrt n}+\frac1n+p_n+\frac1{2\sqrt m}\right)
  \longrightarrow0.
\end{equation}
Finally, a continuous $f$ on the compact state cube is uniformly continuous.
Choose $\delta>0$ so that $\|y-z\|_\infty\le\delta$ implies
$|f(y)-f(z)|<\epsilon$. Then
\[
  |\E f(G^{m,n}(x))-f(G(x))|
  \le\epsilon+2\|f\|_\infty
       \Pr\!\left(\|G^{m,n}(x)-G(x)\|_\infty>\delta\right).
\]
Taking the supremum over $\mathcal X^{m,n}\cap\Xg$, then the joint limit and
$\epsilon\downarrow0$, proves the second assertion.
\end{proof}
}

\begin{revblock}
\begin{proof}[Proof of Lemma~\ref{lem:mn_reward}]
  Let $\Xg=\{x\in\mathcal{X}:\sum_r x_{r,0}\ge k+\gamma/2\}$ be the set of well-stocked states.
  Fix $x\in\mathcal{X}^{m,n}$ and write $\phi=(x_{r,0})_r$. Let $\phi^{(j)}$ be the free fractions that
  dispatcher $j$ observes and $A_j=\mathbf 1[\sum_r\phi^{(j)}_r\ge k]$. Let $\hat a_r/n$ be the fraction of
  type-$r$ resources that \Lag\ assigns to a type-$v$ task, {\color{red}using the capped proportional allocation and largest-remainder rounding of Definition~\ref{def:fluid_scaled} whenever $A_j=1$; set $\hat a=0$ when $A_j=0$.
  On an admitted task, that definition gives a rounding error below $1/n$ in each normalized coordinate.} All dispatchers are identically distributed, so the
  expected reward of one step, divided by $mn$, is
  $\rho^{m,n}(x)=\sum_{v,r}s_{vr}\lambda_v\,\E[A_1\hat a_r/n]$. Write
  $\rho(x)=\sum_{v,r}s_{vr}u^{\Lag}_{vr}(x)$.

  Suppose $x\in\Xg$. A waterfill allocates at most its budget $k$. The Lipschitz bound for the
  waterfill (Lemma~\ref{lem:waterfill_lipschitz}, constant $4$),
  the rounding error, and $\E[1-A_1]\le p_n$ give
  \[
    |\rho^{m,n}(x)-\rho(x)|
    \le s_{\max}\sum_v\lambda_v\left(4\,\E\|\phi^{(1)}-\phi\|_1+\frac{R}{n}+k\,\E[1-A_1]\right)
    \le s_{\max}\left(\frac{4R}{\sqrt n}+\frac Rn+k\,p_n\right).
  \]
  Here $\E\|\phi^{(1)}-\phi\|_1\le R/\sqrt n$ by~\eqref{eq:hypergeo-moments}, $\sum_v\lambda_v\le1$, and
  $\E[1-A_1]\le p_n$ holds because a shortage at a dispatcher requires some $x_{r,0}-\phi^{(1)}_r>\gamma/(2R)$,
  which Hoeffding's inequality for sampling without replacement bounds by $p_n$ (as in the proof of
  Proposition~\ref{prop:congestion}). If $x\notin\Xg$, then $|\rho^{m,n}(x)-\rho(x)|\le 2k\,s_{\max}$.

  Since $\mathcal{X}^{m,n}$ is finite, the Ces\`aro averages $\frac1T\sum_{t\le T}P^t(x(0),\cdot)$ converge to a
  stationary distribution $\mu^{m,n}$~\cite[Eq.~(5.6.1)]{durrett2019probability}, and
  $\Reward_{m,n}(\Lag)/(mn)=\int\rho^{m,n}\,d\mu^{m,n}$. Therefore
  \[
    \left|\int\rho^{m,n}\,d\mu^{m,n}-\int\rho\,d\mu^{m,n}\right|
    \le s_{\max}\left(\frac{4R}{\sqrt n}+\frac Rn+k\,p_n\right)+2k\,s_{\max}\,\mu^{m,n}(\mathcal{X}\setminus\Xg).
  \]
  By Lemma~\ref{lem:free-conc}, $\mu^{m,n}(\mathcal{X}\setminus\Xg)\le e^{-\gamma^2m/(2k^2d)}$.
\end{proof}
\end{revblock}

\begin{proof}[Proof of Lemma~\ref{lem:free-conc}]
Each in-service task occupies $nk$ servers and there are $mn$ servers of each type, so the
busy fraction is $R - \sum_r x_{r,0} = \tfrac{nk}{mn}\,N = \tfrac{k}{m}\,N$, where $N$ is
the number of tasks in service. Since $R(1-\rho) = k+\gamma$, the event
$\sum_r x_{r,0}\le k+\tfrac{\gamma}{2}$ is exactly $\tfrac{k}{m}N \ge \rho R + \tfrac{\gamma}{2}$.

Index arrival slots by $(j,s)$ with dispatcher $j\in[m]$ and age $s\in\{0,\dots,d-1\}$. In
the loss model a served task starts service at its arrival, so a task is in service in the
current state iff it was accepted at some slot $(j,s)$ and its duration exceeds its age,
$D>s$. Dropping the acceptance indicator,
\[
  N \;\le\; N_\infty \;:=\; \sum_{j=1}^m\sum_{s=0}^{d-1}
    \mathbf 1[\text{slot }(j,s)\text{ holds a task with } D>s].
\]
The $md$ indicators are independent (distinct slots are distinct arrivals, with independent
types and durations), and slot $(j,s)$ is active with probability
$\sum_v \lambda_v\Pr[D_v>s]$. Hence
\[
  \E[N_\infty] = m\sum_{s=0}^{d-1}\sum_v \lambda_v\Pr[D_v>s] = m\sum_v\lambda_v\E[D_v],
  \qquad \tfrac{k}{m}\E[N_\infty] = k\sum_v\lambda_v\E[D_v] = \rho R.
\]
By Hoeffding's inequality over these $md$ independent indicators in $\{0,1\}$,
\[
  \Pr\Bigl[\tfrac{k}{m}N \ge \rho R+\tfrac{\gamma}{2}\Bigr]
  \;\le\; \Pr\Bigl[N_\infty - \E[N_\infty] \ge \tfrac{\gamma m}{2k}\Bigr]
  \;\le\; \exp\!\Bigl(-\tfrac{\gamma^2 m}{2k^2 d}\Bigr) \;\xrightarrow[m\to\infty]{}\; 0.
\]
Arrivals and durations are exogenous, independent of the policy and of the state history, so
this bound holds under the stationary law $\mu^{m,n}_\pi$ of any work-conserving policy
$\pi$, giving the claim.
\end{proof}

\begin{proof}[Proof of Proposition~\ref{prop:congestion}]
Fix a type $v$ and let $\gamma$ be as in Lemma~\ref{lem:free-conc}. By that lemma
the system's free capacity exceeds $k + \gamma/2$ with
probability tending to one, so it suffices to bound the probability on that
event. There the task is short only if its dispatcher's free capacity falls short of
the system's by more than $\gamma/2$, i.e.\ $\sum_r |X^{(1)}_{r,0}-x_{r,0}| >
\gamma/2$. Each $X^{(1)}_{r,0}$ is a sample of $n$ servers drawn without
replacement, so by the Serfling inequality~\cite{serfling1974probability} (as in
Lemma~\ref{lem:mn_transition}) and a union bound this has probability at most
$2R\,e^{-n(\gamma/2R)^2}$, uniformly in $x$. Both contributions vanish; averaging
over $v$ gives $P^{\mathrm{block}}_{m,n}\to0$.
\end{proof}

\begin{proof}[Proof of Lemma~\ref{lem:4_mu_convergence}]
  We show for every bounded continuous $f$ that $\int f\,d\mu^{m,n}\to f(x^*)$,
  following the approach of~\cite{benaim2008class}.

  First, every subsequence of $\{\mu^{m,n}\}$ has a weakly convergent sub-subsequence.
  Each $\mu^{m,n}$ is supported on
  $\mathcal{X}^{m,n}\subseteq[0,1]^{R\times(d+1)}$. Since $[0,1]^{R\times(d+1)}$ is compact,
  by Prokhorov's theorem~\cite[Theorem~5.1]{billingsley2013convergence} every subsequence of
  $\{\mu^{m,n}\}$ has a weakly convergent sub-subsequence.

  Second, every subsequential limit is $G$-invariant.
  Let $\mu^{m_{k_j},n_{k_j}}\Rightarrow\mu$. Stationarity of $\mu^{m_{k_j},n_{k_j}}$ gives:
  \[
    \int f\,d\mu^{m_{k_j},n_{k_j}}
    = \int\E\!\left[f\!\left(G^{m_{k_j},n_{k_j}}(x)\right)\right]d\mu^{m_{k_j},n_{k_j}}.
  \]
  {\color{red}The LHS converges to $\int f\,d\mu$. For the RHS, define
  \[
    \epsilon_{m,n}=\sup_{x\in\mathcal X^{m,n}\cap\Xg}
      |\E[f(G^{m,n}(x))]-f(G(x))|\longrightarrow0
  \]
  by Lemma~\ref{lem:mn_transition}. Splitting the stationary integral over $\Xg$
  and its complement gives
  \begin{align*}
    &\left|\int\E[f(G^{m,n}(x))]d\mu^{m,n}(x)
             -\int f(G(x))d\mu^{m,n}(x)\right|\\
    &\qquad\le\epsilon_{m,n}
       +2\|f\|_\infty\,\mu^{m,n}(\mathcal X\setminus\Xg)\\
    &\qquad\le\epsilon_{m,n}
       +2\|f\|_\infty e^{-\gamma^2m/(2k^2d)}\longrightarrow0,
  \end{align*}
  using Lemma~\ref{lem:free-conc}. Since $G$ is continuous by
  Lemma~\ref{lem:waterfill_lipschitz}, weak convergence implies
  \[
    \int f(G(x))d\mu^{m_{k_j},n_{k_j}}(x)\to\int f(G(x))d\mu(x).
  \]
  Thus $\int f\,d\mu=\int f(G(x))d\mu(x)$, so $\mu$ is $G$-invariant.}

  Third, the unique $G$-invariant measure is $\delta_{x^*}$.
  For any $G$-invariant probability measure $\mu$, induction gives
  $\int f\,d\mu=\int f(G^t(x))\,d\mu(x)$ for all $t\geq 1$. By
  Assumption~\ref{assum:gap}, $G^t(x)\to x^*$ pointwise, so by dominated convergence:
  \[
    \int f\,d\mu = \lim_{t\to\infty}\int f(G^t(x))\,d\mu(x) = f(x^*).
  \]
  Hence $\mu=\delta_{x^*}$. Since every subsequential limit equals $\delta_{x^*}$, the full
  sequence converges weakly to $\delta_{x^*}$.
\end{proof}

\begin{proof}[Proof of Theorem~\ref{thm:asymptotic_optimality}]
  By Lemma~\ref{lem:mn_reward} and the continuity of $x\mapsto u_{vr}^{\Lag}(x)$
  (Lemma~\ref{lem:waterfill_lipschitz}):
  \[
    \frac{\Reward_{m,n}(\Lag)}{mn}
    = \sum_{v,r}s_{vr}\int u_{vr}^{\Lag}(x)\,d\mu^{m,n}(x) + {\color{red}o(1)}
    \;\xrightarrow[\mu^{m,n}\Rightarrow\delta_{x^*}]{}\; \sum_{v,r}s_{vr}\,u_{vr}^{\Lag}(x^*),
  \]
  where the weak convergence is by Lemma~\ref{lem:4_mu_convergence}. It remains to show
  $u^{\Lag}(x^*)=u^*$.

\begin{revblock}
  Let $(u^*,\phi^*)$ and $\hat{x}$ be as in Lemma~\ref{lem:lp-optimal}. By Lemma~\ref{lem:lp-optimal}(b),
  $\hat x$ is a fixed point of $G$ and $u^{\Lag}(\hat x)=u^*$. By Assumption~\ref{assum:gap} the fixed point is
  unique, so $x^*=\hat{x}$ and $u^{\Lag}(x^*)=u^*$. Therefore the achievability bound holds:
  \[
    \frac{\Reward_{m,n}(\Lag)}{mn} \;\to\; \sum_{v,r}s_{vr}\,u^*_{vr}
    \;=\;\mathrm{OPT}(\text{OCC-LP}(0)).
  \]
\end{revblock}

  For the matching upper bound, fix any work-conserving policy $\pi$. Applying
  Proposition~\ref{prop:lp_upper_bound} to the $(m,n)$-scaled instance gives
  $\Reward_{m,n}(\pi)/mn \le \mathrm{OPT}(\text{OCC-LP}(B_{m,n}))$, where $B_{m,n}$ is the
  blocking probability of the scaled system. Blocking is exactly the event of
  Proposition~\ref{prop:congestion}, an arriving task finding fewer than $k$ free servers,
  so $B_{m,n} = P^{\mathrm{block}}_{m,n} \to 0$.

  It remains to show $\mathrm{OPT}(\text{OCC-LP}(B_{m,n})) \to \mathrm{OPT}(\text{OCC-LP}(0))$.
  Write $V(b)$ for the optimal value of the LP obtained by replacing the demand
  right-hand side in~\eqref{lp:demand} with a parameter $b\in\mathbb{R}^V$, so that
  $\mathrm{OPT}(\text{OCC-LP}(B)) = V\bigl(k\lambda(1-B)\bigr)$, where
  $k\lambda(1-B) = (k\lambda_v(1-B))_{v\in[V]}$. The parameter $b$ enters only the
  constraint right-hand sides, so $V$ is concave and finite on the set of $b$ for which the
  LP is feasible, hence continuous on the interior of that set. Strict underload
  (Assumption~\ref{assum:underload}) gives $R(1-\rho) > k$, so $\text{OCC-LP}(0)$ admits a
  feasible point with strict slack in every availability constraint~\eqref{lp:bound};
  perturbing $b$ near $k\lambda$ preserves feasibility, so $b = k\lambda$ lies in the
  interior. Since $k\lambda(1-B_{m,n}) \to k\lambda$ as $B_{m,n}\to 0$, continuity of $V$ at
  the interior point $k\lambda$ gives
  $\mathrm{OPT}(\text{OCC-LP}(B_{m,n})) \to \mathrm{OPT}(\text{OCC-LP}(0))$.

  Hence $\limsup_{m,n\to\infty}\Reward_{m,n}(\pi)/mn \le \mathrm{OPT}(\text{OCC-LP}(0))$,
  matching the achievability bound, and \Lag\ is asymptotically optimal.
\end{proof}

\section{Diagrams of the Two Partial Scalings}\label{app:scaling-figs}

\begin{figure}[!ht]
    \centering
\begin{subfigure}[t]{\linewidth}
\centering
\includegraphics[width=0.5\textwidth]{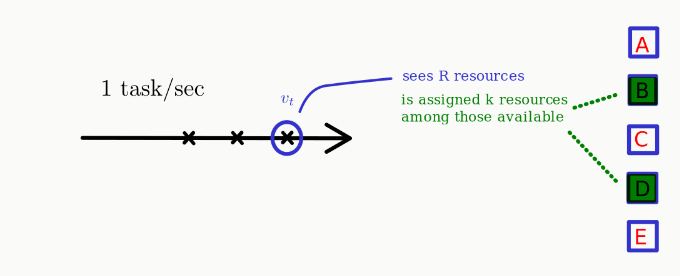}
\caption{Original system}
\label{fig:scaling:orig}
\end{subfigure}
\medskip

\begin{subfigure}[t]{\linewidth}
\centering
        \includegraphics[width=0.55\linewidth]{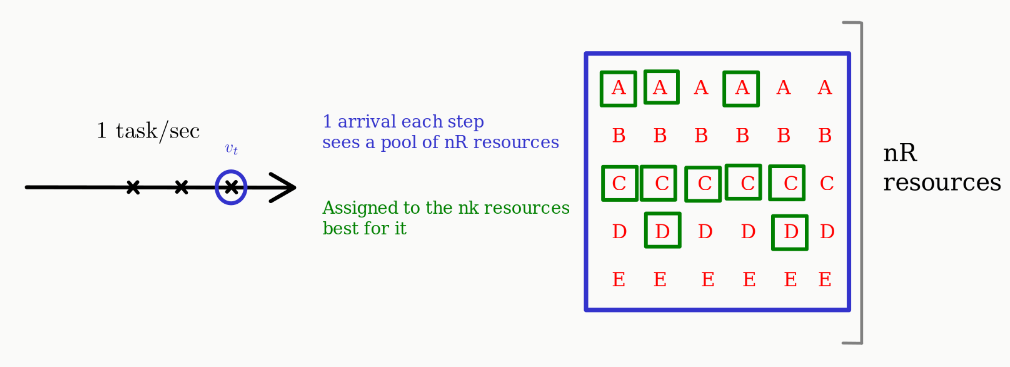}
        \caption{$n$-scaled system}
        \label{fig:n_scaled_system}
    \end{subfigure}
    \medskip
\begin{subfigure}[t]{\linewidth}
\centering
\includegraphics[width=0.5\textwidth]{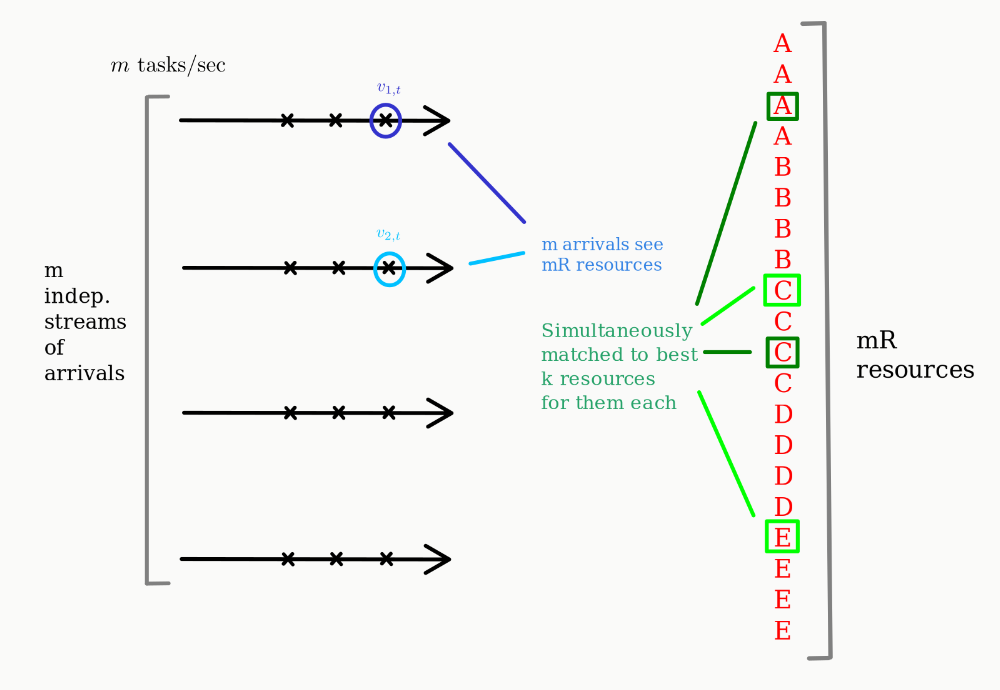}
\caption{$m$-scaled system}
\label{fig:scaling:m}
\end{subfigure}
\caption{Two natural scalings of the original system: (a) the original problem; (b) $n$-scaling, which keeps one random arrival per step; (c) $m$-scaling, which lets resources compare across a batch of simultaneous arrivals.}
\label{fig:scaling}
\end{figure}

\section{Full Load-Sweep Results}\label{app:full_results}

Table~\ref{tab:perf_full} reports reward as a percentage of OCC-UB for all policies
across the full $\rho$ sweep, for each of the six instances.

\begin{longtable}{llrrrrr}
  \caption{Reward as \% of OCC-UB across load $\rho$, all instances.}
  \label{tab:perf_full} \\
  \toprule
  Instance & $\rho$ & \Lag & Whittle & ALG-SC-LP & Greedy & Random \\
  \midrule
  \endfirsthead
  \toprule
  Instance & $\rho$ & \Lag & Whittle & ALG-SC-LP & Greedy & Random \\
  \midrule
  \endhead
  \midrule
  \multicolumn{7}{r}{\textit{continued on next page}} \\
  \endfoot
  \bottomrule
  \endlastfoot
  \multirow{9}{*}{TMLR}
    & 0.1 & 98.9 & 79.3 & 74.8 & 98.2 & 36.5 \\
    & 0.2 & 98.1 & 82.7 & 78.4 & 95.8 & 38.8 \\
    & 0.3 & 97.1 & 84.8 & 76.2 & 92.6 & 40.9 \\
    & 0.4 & 94.6 & 85.2 & 85.2 & 88.2 & 43.1 \\
    & 0.5 & 92.8 & 85.6 & 83.0 & 84.7 & 45.2 \\
    & 0.6 & 90.8 & 85.2 & 81.6 & 81.6 & 47.4 \\
    & 0.7 & 87.8 & 83.2 & 80.5 & 78.4 & 49.6 \\
    & 0.8 & 85.8 & 82.3 & 79.4 & 76.4 & 51.3 \\
    & 0.9 & 85.5 & 82.8 & 78.7 & 76.5 & 53.5 \\
  \midrule
  \multirow{9}{*}{ICLR}
    & 0.1 & 98.8 & 88.0 & 87.7 & 98.3 & 55.6 \\
    & 0.2 & 98.2 & 89.5 & 90.2 & 96.8 & 58.3 \\
    & 0.3 & 98.1 & 91.4 & 89.7 & 95.4 & 61.4 \\
    & 0.4 & 97.0 & 92.1 & 91.8 & 93.0 & 64.3 \\
    & 0.5 & 96.6 & 93.1 & 91.9 & 91.6 & 67.3 \\
    & 0.6 & 95.8 & 93.1 & 91.8 & 90.3 & 70.0 \\
    & 0.7 & 94.5 & 92.5 & 90.0 & 88.9 & 72.0 \\
    & 0.8 & 93.6 & 92.1 & 89.6 & 87.9 & 74.1 \\
    & 0.9 & 93.1 & 92.3 & 89.0 & 87.8 & 75.1 \\
  \midrule
  \multirow{8}{*}{Specialist}
    & 0.1 & 98.8 & 96.9 & 95.9 & 96.9 & 37.8 \\
    & 0.2 & 94.5 & 90.4 & 90.2 & 90.4 & 37.7 \\
    & 0.3 & 88.7 & 83.7 & 84.9 & 83.7 & 37.9 \\
    & 0.4 & 89.1 & 83.4 & 86.4 & 83.4 & 41.5 \\
    & 0.5 & 88.7 & 82.8 & 86.7 & 82.8 & 45.5 \\
    & 0.6 & 88.7 & 83.6 & 86.8 & 83.6 & 49.1 \\
    & 0.7 & 90.6 & 85.5 & 84.6 & 85.5 & 55.0 \\
    & 0.8 & 91.9 & 87.4 & 85.1 & 87.4 & 61.5 \\
  \midrule
  \multirow{8}{*}{Priority}
    & 0.1 & 88.6 & 83.5 & 89.2 & 88.6 & 41.1 \\
    & 0.2 & 87.8 & 87.8 & 80.8 & 81.4 & 44.0 \\
    & 0.3 & 87.8 & 87.8 & 77.5 & 73.5 & 46.4 \\
    & 0.4 & 86.0 & 86.0 & 76.5 & 68.0 & 47.6 \\
    & 0.5 & 86.2 & 86.2 & 77.3 & 65.1 & 50.6 \\
    & 0.6 & 86.0 & 86.0 & 79.6 & 64.0 & 53.0 \\
    & 0.7 & 85.6 & 85.6 & 82.2 & 63.0 & 54.9 \\
    & 0.8 & 83.0 & 83.0 & 82.9 & 62.7 & 56.8 \\
  \midrule
  \multirow{9}{*}{Unfriendly}
    & 0.1 & 98.1 & 98.1 & 97.5 & 97.5 & 26.6 \\
    & 0.2 & 93.7 & 93.7 & 92.5 & 90.0 & 27.4 \\
    & 0.3 & 89.3 & 89.3 & 86.1 & 80.6 & 28.7 \\
    & 0.4 & 88.8 & 88.8 & 84.0 & 75.4 & 32.2 \\
    & 0.5 & 89.2 & 89.2 & 83.6 & 70.9 & 35.5 \\
    & 0.6 & 89.7 & 89.7 & 84.7 & 69.1 & 39.4 \\
    & 0.7 & 88.7 & 88.7 & 82.2 & 66.8 & 42.7 \\
    & 0.8 & 87.8 & 87.8 & 80.0 & 65.2 & 44.7 \\
    & 0.9 & 85.3 & 85.3 & 78.1 & 63.3 & 47.7 \\
  \midrule
  \multirow{9}{*}{Low-rank}
    & 0.1 & 98.8 & 43.0 & 96.0 & 98.4 & 51.7 \\
    & 0.2 & 97.8 & 47.5 & 95.6 & 96.6 & 54.6 \\
    & 0.3 & 96.1 & 51.1 & 94.0 & 94.0 & 55.8 \\
    & 0.4 & 95.5 & 55.8 & 93.5 & 92.3 & 58.6 \\
    & 0.5 & 95.4 & 61.0 & 92.6 & 90.7 & 61.1 \\
    & 0.6 & 94.6 & 66.2 & 92.0 & 88.5 & 63.7 \\
    & 0.7 & 94.3 & 71.9 & 91.8 & 86.3 & 67.1 \\
    & 0.8 & 93.0 & 76.6 & 91.1 & 83.6 & 69.7 \\
    & 0.9 & 91.4 & 80.0 & 89.6 & 81.6 & 71.9 \\
\end{longtable}

\end{document}